\documentclass[conference]{IEEEtran}
\IEEEoverridecommandlockouts
\usepackage{cite}
\usepackage{amsmath,amssymb,amsfonts}
\usepackage{algorithmic}
\usepackage{graphicx}
\usepackage{textcomp}
\usepackage{xcolor}

\usepackage{bm}
\usepackage{bbm}
\usepackage{enumitem}
\usepackage{multirow}
\usepackage{diagbox}
\usepackage{caption}
\usepackage{blindtext}
\usepackage{subcaption}
\usepackage{wrapfig}
\usepackage{url}
\usepackage{soul}
\usepackage{tabu}
\usepackage{booktabs}
\usepackage{bbm}

\usepackage{amsthm}
\newtheorem{definition}{Definition}
\newtheorem{assumption}{Assumption}
\newtheorem{theorem}{Theorem}
\newtheorem{lemma}{Lemma}

\newtheorem{proposition}{Proposition}

\usepackage[normalem]{ulem}
\usepackage{makecell}
\usepackage{url}            
\usepackage{booktabs}       
\usepackage{amsfonts}       
\usepackage{nicefrac}       
\usepackage{microtype}      
\usepackage{xcolor}         
\usepackage{algorithm}
\usepackage{amsmath}
\usepackage{amssymb}
\usepackage{graphicx}
\usepackage{wrapfig}
\usepackage[bottom]{footmisc}
\usepackage{bbding}
\usepackage{makecell}
\usepackage{multirow}
\usepackage{ulem}

\usepackage[colorinlistoftodos,prependcaption,textsize=tiny]{todonotes}

\usepackage[numbers]{natbib}

\def\BibTeX{{\rm B\kern-.05em{\sc i\kern-.025em b}\kern-.08em
    T\kern-.1667em\lower.7ex\hbox{E}\kern-.125emX}}

\begin{document}

\title{Graph Anomaly Detection at Group Level: A Topology Pattern Enhanced Unsupervised Approach}

\author{ \IEEEauthorblockN{Xing Ai\IEEEauthorrefmark{1}, Jialong Zhou\IEEEauthorrefmark{1}, Yulin Zhu\IEEEauthorrefmark{1}, Gaolei Li\IEEEauthorrefmark{2}, Tomasz P. Michalak\IEEEauthorrefmark{3}, Xiapu Luo\IEEEauthorrefmark{1}, Kai Zhou\IEEEauthorrefmark{1}}
\IEEEauthorblockA{xing96.ai@connect.polyu.hk, jialong.zhou@connect.polyu.hk, yulin.zhu@polyu.edu.hk, gaolei\_li@sjtu.edu.cn, \\ tpm@mimuw.edu.pl, daniel.xiapu.luo@polyu.edu.hk, kaizhou@polyu.edu.hk}
\IEEEauthorblockA{\IEEEauthorrefmark{1}\textit{Department of Computing}, \textit{The Hong Kong Polytechnic University}, HKSAR}
\IEEEauthorblockA{\IEEEauthorrefmark{2} \textit{Shanghai Jiao Tong University}, Shanghai, China}
\IEEEauthorblockA{\IEEEauthorrefmark{3} \textit{University of Warsaw}, Poland}}

\maketitle
\pagestyle{plain}

\begin{abstract}
Graph anomaly detection (GAD) has achieved success and has been widely applied in various domains, such as fraud detection, cybersecurity, finance security, and biochemistry. However, existing graph anomaly detection algorithms focus on distinguishing individual entities (nodes or graphs) and overlook the possibility of anomalous groups within the graph. To address this limitation, this paper introduces a novel unsupervised framework for a new task called Group-level Graph Anomaly Detection (Gr-GAD). The proposed framework first employs a variant of Graph AutoEncoder (GAE) to locate anchor nodes that belong to potential anomaly groups by capturing long-range inconsistencies. Subsequently, group sampling is employed to sample candidate groups, which are then fed into the proposed Topology Pattern-based Graph Contrastive Learning (TPGCL) method. TPGCL utilizes the topology patterns of groups as clues to generate embeddings for each candidate group and thus distinct anomaly groups. The experimental results on both real-world and synthetic datasets demonstrate that the proposed framework shows superior performance in identifying and localizing anomaly groups, highlighting it as a promising solution for Gr-GAD. Datasets and codes of the proposed framework are at the github repository https://anonymous.4open.science/r/Topology-Pattern-Enhanced-Unsupervised-Group-level-Graph-Anomaly-Detection.
\end{abstract}

\section{Introduction}



Graph Anomaly Detection (GAD) aims to identify anomalies within graph data. Due to the recent technical advances in graph representation learning~\cite{grlsurvey1, grlsurvey2, grlsurvey3}, GAD has demonstrated significant success in various application domains, such as cybersecurity~\cite{GraphConsis, FraudNE}, finance~\cite{9762926}, and Internet of Things~\cite{sipple2020interpretable}. Deep-learning-based GAD methods can be broadly classified into three levels \cite{gadsurvey1, gadsurvey2, gadsurvey3} based on the types of anomalies being investigated: Node-level (N-GAD), Subgraph-level (Sub-GAD), and Graph-level (G-GAD). Specifically, N-GAD \cite{ONE, CAREGNN, BWGNN} and Sub-GAD focus on identifying individual nodes or a subgraph from a single graph, while G-GAD \cite{OCGTL, iGAD, TUAF} aims to distinguish anomalous graphs from a given set of graphs.






A critical application domain of GAD is finance, where various GAD methods are used to detect financial crimes such as fraud \cite{CAREGNN, DeepFD}, Ponzi schemes \cite{heterponzi}, and money laundering \cite{AMLalert, colladon2017using}. However, while recent deep-learning methods offer several advantages, they primarily concentrate on identifying anomalies at the node or graph levels. This approach is inadequate since financial crimes are typically orchestrated by criminal groups rather than individuals~\cite{smurfAML, autoaudit, antibenford}. Such criminal groups are often well-organized and exhibit specific structures. In particular, it is known~\cite{flowscope,crysmurfing} that, in a transaction graph, nodes (i.e., accounts) involved in criminal activities form groups with distinct topology patterns, such as trees comprising of a fraud leader and multiple subordinates or paths consisting of long transaction flows associated with money laundering. For example, Fig.~\ref{RealCase} illustrates a real-world money laundering group consisting of four companies,~\footnote{The full names of these companies are Hilux Services LP (HS), Polux Management LP (PM), LCM Alliance LLP (LA), and Metastar Invest LLP (MI). 
According to investigations of the Organized Crime and Corruption Reporting Project (OCCRP) 
\url{https://www.occrp.org/en/azerbaijanilaundromat/raw-data/.} 
}
presenting a path of four nodes. 

\begin{figure}[!h]
\centering
\includegraphics[width=0.24\textwidth]{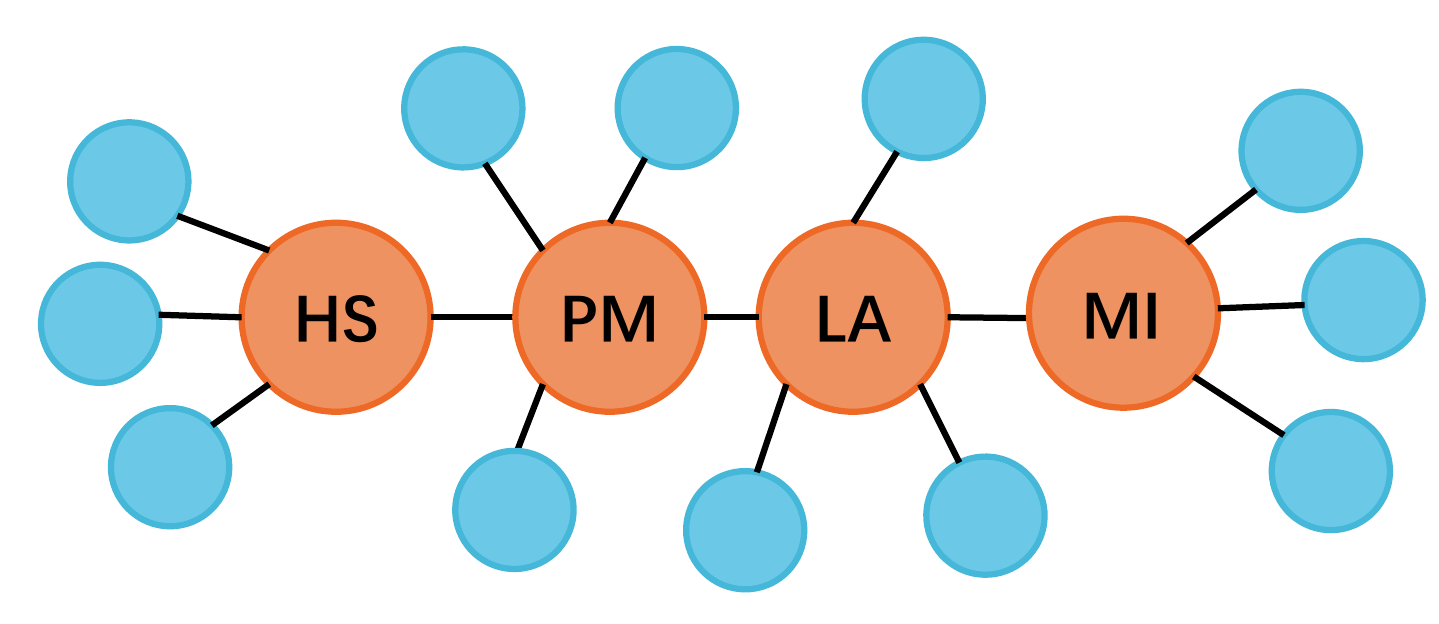}
\caption{A real case of four companies involved in money laundering in Azerbaijan forming a path-like topology pattern. }
\label{RealCase}
\end{figure}



Unfortunately, existing GAD methods (in particular, N-GAD and Sub-GAD) have their own limitations and are thus incapable of detecting anomaly groups with specific topology patterns. Specifically:
\begin{itemize}[leftmargin=5pt]
    \item N-GAD methods treat nodes independently, which neglects topology pattern information. However, existing research~\cite{flowscope, autoaudit} has demonstrated that the topology patterns exhibited by groups are strongly associated with their function or behavior, such as the well-known ``smurfing" structure in money laundering~\cite{crysmurfing}. Therefore, failing to leverage topology information significantly impedes the detection of anomalies at the group level.
    \item N-GAD methods are incapable of detecting anomaly nodes that are deeply embedded within a group. In essence, N-GAD is based on the assumption that abnormal nodes are distinct from their one-hop neighbors. However, nodes deep inside an anomalous group may resemble their one-hop or two-hop neighbors (i.e., internal nodes of the group), while differing from long-range nodes (i.e., external nodes of the group), giving rise to what we refer to as ``long-range inconsistency." We will provide a detailed demonstration of this inefficacy of N-GAD later.
\end{itemize}

It is worth noting that a few emerging Sub-GAD studies~\cite{DeepFD, FraudNE, AS-GAE}
have been developed to detect anomalous subgraphs within a given graph. However, we emphasize that these studies generally utilize N-GAD as the primary component to detect anomalous nodes and assume that the detected nodes constitute an anomalous subgraph. As a result, Sub-GAD shares the same limitations as N-GAD, as outlined above. Lastly, G-GAD aims to distinguish whether an entire graph is anomalous, which addresses a distinct problem from ours.



Given the insufficiency of current GAD approaches, we introduce a new task called Group-level Graph Anomaly Detection (\textbf{Gr-GAD}). Specifically, \textbf{Gr-GAD} aims to address the following problem: give a graph, identify a set of anomalous groups and assign an anomaly score to each group.



Tackling the \textbf{Gr-GAD} task presents several significant challenges. Firstly, the number of groups increases exponentially with the size of the graph, making it computationally infeasible to examine every potential group. Therefore, reducing the number of candidate groups is critical for the efficiency of \textbf{Gr-GAD}. Secondly, effectively utilizing topology patterns information for group detection is an important issue. By fully utilizing the distinctive topology patterns displayed by financial criminal groups, some experience-driven or expert knowledge-based manual methods \cite{flowscope, smurfAML, crysmurfing, antibenford} have successfully located criminal groups. This demonstrates that topology patterns can serve as clues for detecting anomalous groups, thereby enhancing the performance of \textbf{Gr-GAD}. How to universally capture the underlying connections between topology patterns and anomaly groups 
is of paramount importance.

To address the above challenges, we propose a novel unsupervised \textbf{Gr-GAD} framework combining Graph AutoEncoder (GAE) and Graph Contrastive Learning (GCL), namely \textbf{T}opology \textbf{P}attern Enhanced Unsupervised  \textbf{Gr}oup-level \textbf{G}raph \textbf{A}nomaly \textbf{D}etection (TP-GrGAD).

Specifically, we employ the proposed Multi-Hop Graph AutoEncoder (MH-GAE) to detect anchor nodes that potentially belong to anomaly groups. Subsequently, the group sampling phase is initiated, where candidate groups are sampled starting from the anchor nodes, effectively reducing the number of groups that need to be identified. The Topology Pattern Enhanced Graph Contrastive Learning (TPGCL), introduced in this paper, operates on the candidate groups as input and utilizes two novel augmentations to perturb the topology patterns. This process generates embeddings containing topology pattern information for each candidate group, which are then fed into the anomaly detector for anomaly scoring. The main contributions of this paper are summarized as follows.



\begin{itemize}[leftmargin=*]
\item We initiate the study of a new task in the realm of anomaly detection: Group-level Graph Anomaly Detection (Gr-GAD). This task focuses on identifying anomalies at the group level by leveraging topology pattern information, and it has significant applications, particularly in the financial sector.
\item We propose an unsupervised framework comprising a Multi-Hop Graph AutoEncoder (MH-GAE) that can effectively capture long-range inconsistencies and a Topology Pattern-based Graph Convolutional Learning (TPGCL) that can extract latent correlations between an anomaly group and its topology patterns. The proposed framework can detect anomaly groups with arbitrary sizes without any labels.

\item In order to provide insight into the proposed Topology Pattern-based Graph Convolutional Learning (TPGCL), we offer a comprehensive theoretical analysis from the perspective of Graph Information Bottleneck (GIB), which serves to substantiate its effectiveness.

\item Experiments conducted on both real-world datasets and synthetic datasets demonstrate that the proposed framework achieves substantial performance improvements in comparison to the existing N-GAD and Sub-GAD methods. 
\end{itemize}

The remainder of this paper is organized as follows. Sec.~\ref{related work} and Sec.~\ref{Preliminaries} provide overviews of the related work and background. 
Sec.~\ref{Problem Statement} defines the primary research problem of this paper: Gr-GAD. Sec.~\ref{Methodology} comprehensively introduces the proposed framework, and Sec.~\ref{Theoretical Supports} provides theoretical analysis. Sec.~\ref{Experiments} describes our experimental setting and demonstrates empirically the performance of the proposed framework. Finally, Sec.~\ref{Conclusion} concludes the paper.

\section{Related Work}\label{related work}



\subsection{Graph Anomaly Detection}

Graph anomaly detection is rooted in the body of research on general anomaly detection, excellent overviews of this literature can be found in the works by Chandola et al.~\cite{chandola2009anomaly}, Shubert et al.~\cite{schubert2014local}, Wang et al.~\cite{wang2019progress}, Boukerche et al.~\cite{boukerche2020outlier}, and, recently, Samariya et al.~\cite{samariya2023comprehensive}. In turn, a well-known survey by Akoglu et al.~\cite{gadsurvey1} focuses on classic (non-deep learning) techniques for graph anomaly detection. Nevertheless, with the introduction of various advanced deep learning technologies, such as Graph Neural Networks (GNN) and Graph AutoEncoder (GAE), a lot of graph anomaly detection methods achieve state-of-the-art in different application scenarios. For example, Tang et al.~\cite{BWGNN} find the ‘right-shift’ phenomenon of spectral energy distributions when applying GNN to anomaly detection and propose Beta Wavelet Graph Neural network (BWGNN), which is a better spectral to capture graph anomaly. Liu et al.~\cite{CoLA} indicate GAE can not exploit the rich local information well since its learning objective is reconstructing the whole graph, and propose a contrastive self-supervised learning framework. Unlike widely researched node-level anomaly detection, Chen et al.~\cite{OCGTL} notice the graph-level anomaly detection and propose a novel graph-level method based on One-Class Classification (OCC).

Compared to well-studied node-level/graph-level anomaly detection, Subgraph Anomaly Detection (Sub-GAD) receives much less attention. We currently know three deep learning-based subgraph anomaly detection methods: DeepFD~\cite{DeepFD}, FraudNE~\cite{FraudNE} and AS-GAE~\cite{AS-GAE}, which employ GAE as the backbone to detect anomalous nodes first and then extract anomalous subgraph from anomalous nodes via clustering or connected component detection. Overall, these methods generally follow the style of node-level anomaly/outlier detection, locating anomalous nodes by measuring the consistency of a node with its one-hop neighbors. 


\subsection{Graph Contrastive Learning}

Since the seminal work in \cite{SimCLR}, a diverse set of contrastive learning models have been proposed and achieved impressive results in various domains. GRACE~\cite{GRACE} introduced contrastive learning into the graph domain. It perturbs the input graph randomly, such as edge removal or feature masking, to generate views and maximize the mutual information between the input graph and views. Assuming a graph and nodes within a graph should share similar representations, MVGRL~\cite{MVGRL} proposes a local-global contrastive loss that compares the similarity between nodes' representation and the whole graph's. BGRL~\cite{BGRL}  predicts possible views by introducing bootstrapped graph latent training. Following Graph Information Bottleneck (GIB)~\cite{GIB}, AD-GCL~\cite{AD-GCL} and VGIB~\cite{VGIB} aim to decrease the mutual information between the input graph and its representation while increasing the mutual information between the graph's representation and labels or downstream tasks, thereby reducing irrelevant redundant information with the labels. To achieve this, AD-GCL~\cite{AD-GCL} introduces adversarial training, while  VGIB~\cite{VGIB} incorporates variational approaches.

\section{Preliminaries}\label{Preliminaries}

\subsection{GAE for Node-level Anomaly Detection}
The core assumption of GAE-based node-level anomaly detection is that anomalous nodes in a graph often display inconsistent structural or attribute characteristics compared to their neighbors~\cite{ONE, DONE}. Specifically, if a node is linked to nodes belonging to different classes or communities than itself, it is deemed a structural anomaly. Conversely, if a node's attributes resemble those of nodes from different classes or communities, it is considered an attribute anomaly.

GAE-based anomaly detection utilizes a \textit{reconstruction error} to quantify such a local inconsistency. Specifically, a GNN-based encoder $f_{enc}$ is used to generate low-dimensional embeddings for nodes: $z = f_{enc}(A, X)$, where $A$ and $X$ are the adjacency and feature matrix of the graph, respectively. Then, from the generated embeddings $z$, a decoder $f_{dec}$ is used to reconstruct the graph structure $A^{'}$ and attributes $X^{'}$: $A^{'}, X^{'} = f_{dec}(z)$.



For an arbitrary node $i$ in the graph, its reconstruction error $r_i$ is a weighted sum of a structure error $r_{stru}$ and an attribute error $r_{attr}$:
\begin{equation}\label{GAE stru error}
    \begin{aligned}
        & r_i = \lambda \cdot r_{stru} + (1-\lambda) \cdot r_{attr}, \\
        & r_{stru} = \sum_{j\in N(i)}||A_{ij}-A^{'}_{ij}||, 
        & r_{attr} = ||x_i-x^{'}_i||,\\
    \end{aligned}
\end{equation}
where $N(i)$ is the neighbor set of node $i$ (not including $i$) 
and $\lambda \in [0, 1]$ is a hyperparameter that tunes the relative importance between the structure and attribute error.
The sum of the reconstruction errors of all nodes $\mathcal{L} = \sum_{i\in V}r_i$ is then minimized to jointly train the encoder and decoder. 
After training, the nodes with larger reconstruction errors (e.g., exceeding a threshold value $\tau$)  are identified as anomalous. 

\subsection{Graph Contrastive Learning}




Recently, several unsupervised graph learning methods have emerged, and Graph Contrastive Learning (GCL) is a powerful and widely used approach among them. The GCL model generates embeddings for downstream tasks in the absence of available labels. Early GCL is founded on maximizing mutual information (InfoMax)~\cite{InfoMax}. This is achieved by generating multiple views of an input graph through augmentations and maximizing the mutual information between these views and the input graph, as the training objective. The common augmentations employed in GCL involve perturbations to nodes and edges, such as adding or removing nodes/edges.

Recent GCL methods such as AD-GCL~\cite{AD-GCL} and VGIB~\cite{VGIB} follow the Graph Information Bottleneck (GIB) principle. They aim to simultaneously maximize the mutual information between the input graph's encodings and labels while minimizing the mutual information between the encodings and the input graph. This can be expressed in the following equation:
\begin{equation} \label{IB}
    \max_\theta I(Y; f(G)) - I(G; f(G)),
\end{equation}
where $G$ and $Y$ are input graphs and labels respectively and $f$ is a model such as GNN. Due to the unavailability of labels $Y$, AD-GCL~\cite{AD-GCL} provides the bound of GIB and minimizes Eqn.~\eqref{IB} by introducing regularization into the objective function.





\section{Problem Statement}\label{Problem Statement}
In this section, we formally define the problem of Group-level Graph Anomaly Detection (\textbf{Gr-GAD}) and differentiate it from Sub-GAD. 

We denote a graph as $\mathcal{G}=(\mathbf{V}, \mathbf{E})$, where $\mathbf{V}$ is the set of nodes and $\mathbf{E}$ is the set of edges, and all the node attributes are summarized as a matrix $X$. 


\begin{definition}
    Given a graph $\mathcal{G}=(\mathbf{V}, \mathbf{E})$ and threshold $\tau$ as input, a \textbf{Gr-GAD} method can be represented as a function $F$ that outputs a set of groups $\mathcal{C}$ associated with anomaly scores $\mathcal{S}$: $F(\mathcal{G}) \rightarrow \{\mathcal{C}, \mathcal{S}\}$, where $\mathcal{C}$ and $\mathcal{S}$ fulfill: 
    \begin{equation*}
    \begin{aligned}
            a) \quad \mathcal{C}= & \{c_{1}=(\mathbf{V}_1, \mathbf{E}_1), \cdots,c_{m}=(\mathbf{V}_m, \mathbf{E}_m)\}, \\
            & \quad \mathbf{V}_{i} \subseteq \mathbf{V}, \mathbf{E}_{i} \subseteq \mathbf{E}, \quad  \forall c_i, c_j\in \mathcal{C}, \\
            b) \quad \mathcal{S}= & \{s_1, s_2, \cdots, s_m\}, \quad s_i > \tau, \ \forall s_i \in \mathcal{S},\\
    \end{aligned}
    \end{equation*} 
where $s_i$ is the anomaly score associated with the group $c_i$.
\end{definition}

A good function $F$ should output relatively high anomaly scores for those identified groups. Thus, we assume that each anomaly score $s_i$ should exceed a threshold value $\tau$. We also emphasize that there is no label information provided as input. This paper focuses on unsupervised \textbf{Gr-GAD}.



It is worth noticing the distinction between \textbf{Gr-GAD} and Sub-GAD. In particular, Sub-GAD outputs a subgraph $\mathcal{G}_{sub} = \{\mathbf{V}_{sub}, \mathbf{E}_{sub}\}$ induced by the node set $\mathbf{V}_{sub}$, together with an anomaly score $s_i >\tau$ for \textit{each node} $v_i \in \mathbf{V}_{sub}$. Thus, there are two major distinctions between \textbf{Gr-GAD} and Sub-GAD:
\begin{itemize}[leftmargin=*]
    \item Sub-GAD produces anomaly scores at the node level and the identified anomalous nodes constitute a single anomalous subgraph. In comparison,  \textbf{Gr-GAD} assigns anomaly scores at the group level, that is, assigning anomaly scores to groups instead of nodes.
    \item In the anomalous subgraph $\mathcal{G}_{sub}$ produced by Sub-GAD, \textit{every} node is anomalous from the perspective of the node level (i.e.,  $s_i >\tau, \forall v_i \in \mathbf{V}_{sub}$). In contrast, nodes in an anomalous group identified by \textbf{Gr-GAD} may be normal at the node level. That is, \textbf{Gr-GAD} distinguishes an anomalous group from others at the group level, whereas Sub-GAD identifies a subgraph comprised of nodes that are deemed abnormal at the individual node level.
\end{itemize}


\begin{figure*}[!t]
    \centering
    \includegraphics[scale=0.35]{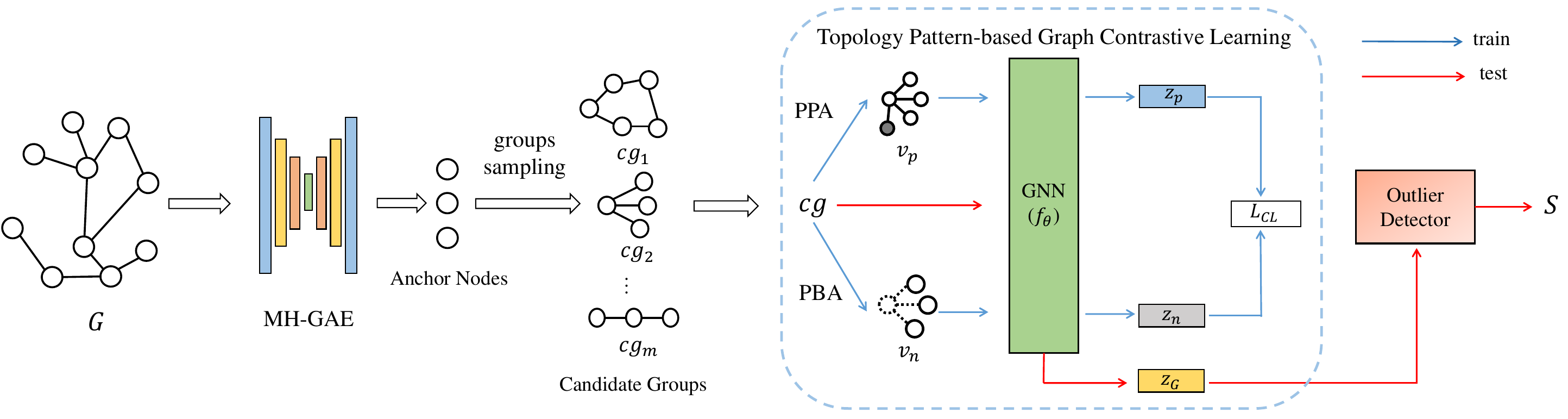}
    \caption{Framework of the TP-GrGAD.}
    \label{Framework}
\end{figure*}


\section{Methodology}\label{Methodology}
In this section, we introduce our proposed model \textbf{TP-GrGAD} to address the task of \textbf{Gr-GAD}. We provide an overview (Sec.~\ref{sec-overview}), followed by detailed explanations of three components: anchor node location via Multi-Hop GAE (Sec.~\ref{MH-GAE section}), candidate groups sampling (Sec.~\ref{sec-sample}), and candidate group discrimination via Topology Pattern-based Graph Contrastive Learning (TPGCL) (Sec.~\ref{sec-gcl}).

\subsection{Overview}
\label{sec-overview}


Directly identifying anomalous groups becomes challenging as the expansive time complexity. To tackle this issue, we adopt a three-step approach that first identifies a set of anchor nodes as potential anomalies. Subsequently, we sample and distinguish groups based on these anchor nodes. The overall framework of our approach \textbf{TP-GrGAD} is presented in Fig.~\ref{Framework}.

\textbf{TP-GrGAD} is composed of three primary components: anchor node localization, candidate group sampling, and candidate group classification.  
To overcome the limitations of GAE, we introduce a novel variant called Multi-Hop Graph AutoEncoder (MH-GAE), which can more accurately detect nodes within anomalous groups by capturing \textit{long-range inconsistencies}. 
We then propose a group sampling algorithm that starts from the anchor nodes to sample candidate groups. 
Finally, given the sampled candidate groups, we propose a novel GCL framework termed  Topology Pattern-based Graph Contrastive Learning (TPGCL) that is capable of capturing topology pattern information by comparing the positive sample ($v_p$ in Fig.~\ref{Framework}) and the negative sample ($v_n$ in Fig.~\ref{Framework}). This enables TPGCL to effectively generate embeddings for the candidate groups, which are then used to identify the presence of anomalous groups via unsupervised outlier detectors such as SUOD \cite{SUOD} and ECOD \cite{ECOD}.



\subsection{Anchor Node Localization by MH-GAE}
\label{MH-GAE section}
To address the challenge of the exponential growth of possible anomalous groups with the increasing graph size, we identify a set of anchor nodes via the proposed Multi-Hop Graph AutoEncoder (MH-GAE), and subsequently sample groups based on these anchor nodes. Below, we begin with a discussion on the limitation of vanilla GAE:

\subsubsection{Limitation of GAE}


Despite the impressive performance achieved by Graph AutoEncoder (GAE) and its variants in unsupervised anomaly detection, it is important to note these methods face limitations when it comes to detecting anomalous groups. Specifically, some anomaly nodes are consistent with their neighbors within the same group but exhibit inconsistency with other long-range nodes outside the group, which we term as \textbf{long-range inconsistency}. Unfortunately, existing GAE-based anomaly detection methods fail to capture such a long-range inconsistency and thus fail to detect group-level anomalies, as demonstrated in Fig.~\ref{example of long range}. In this figure, a given graph contains three anomaly groups denoted in three colors. Although a typical GAE-based node anomaly detection method, DOMINANT~\cite{dominant}, can detect most of the anomalous nodes (shown in red), it fails to identify a few nodes deep inside the anomaly groups (blue, orange, and green in  Fig.~\ref{dominant example}).

\begin{figure}[!ht]
\centering
\begin{subfigure}{0.23\textwidth}
    \centering
    \includegraphics[width=3.2cm, height=2.5cm]{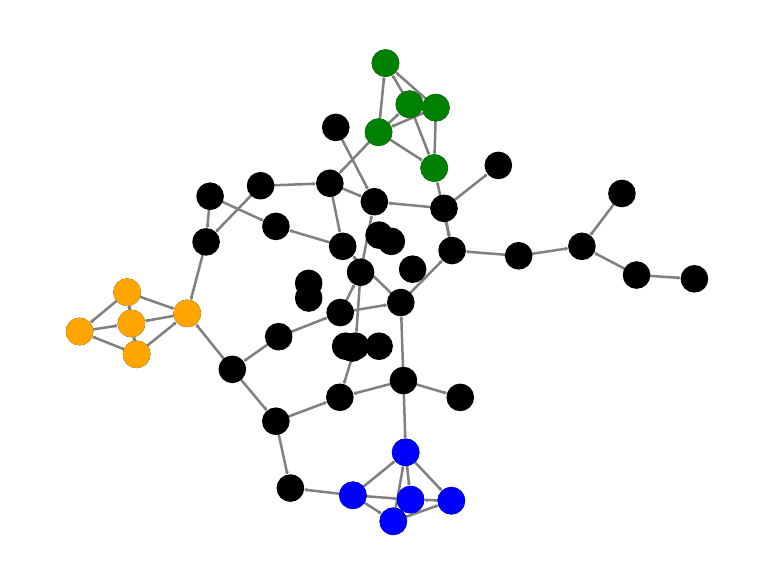}
    \caption{}
\end{subfigure}\hspace{2mm}
\begin{subfigure}{0.23\textwidth}
    \centering
    \includegraphics[width=3.2cm, height=2.5cm]{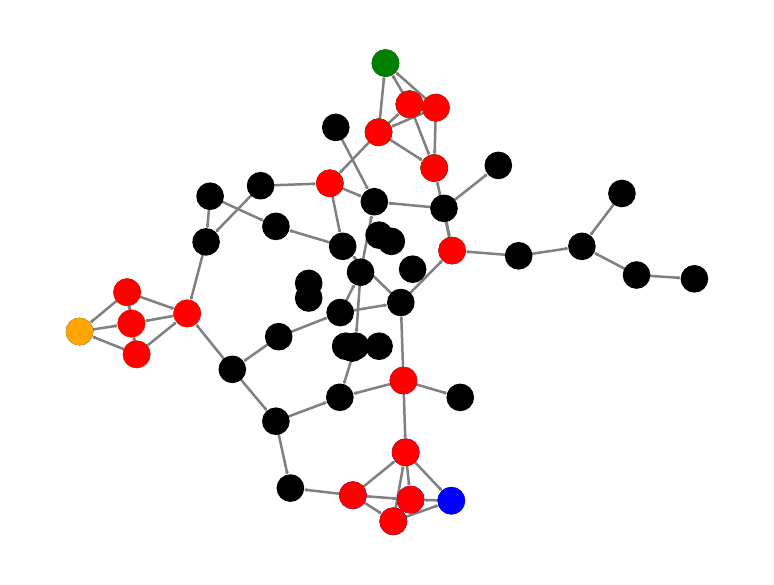}
    \caption{}
    \label{dominant example}
\end{subfigure}
\caption{(a) An example graph contains three anomaly groups, highlighted with green, orange and blue. (b) Predicted anomaly nodes by DOMINANT~\cite{dominant}(red color).}
\label{example of long range}
\end{figure}

This limitation can be primarily attributed to GAE's reconstruction objective, which is the adjacent matrix $A$, as shown in Eqn.~\eqref{GAE stru error}. GAE-based anomaly detection methods emphasize measuring the dissimilarity between nodes and their immediate one-hop neighbors. As a result, identifying nodes that are deeply embedded within an anomalous group, which exhibits similarities to their one-hop neighbors, can pose a challenge for detection using GAE.


\subsubsection{Multi-Hop GAE}
To address the limitation, we propose a novel  Multi-Hop Graph AutoEncoder (MH-GAE), which focuses on the inconsistencies between nodes and their multi-hop neighbors, thereby capturing long-range inconsistencies.

To this end, we explore two approaches. A naive approach is replacing the adjacent matrix $A$ in the objective function (Eqn.~\eqref{GAE stru error}) with a standardized multi-hop matrix such as $A^2$, $A^3$, et al. For a $k$-hop matrix $A^k$, which is standardized $k-$th power of $A$, the reconstruction error of node $i$ is:

\begin{equation} \label{MH-GAE stru error}
    r_{stru} = \sum_{j\in \mathbf{V}}||A^{k}_{ij}-{A^{'}_{ij}}^{k}||.
\end{equation}
Eqn.~\eqref{MH-GAE stru error} regards $A^k$ as the reconstruction objective and asks the model to reconstruct not only the substructure within one-hop neighbors but also within $k-$hop neighbors, thus capturing long-range inconsistency and aware nodes deep in a group that is neglected by vanilla GAE. 

The lack of flexibility is an important consideration in this approach due to the challenge of determining the appropriate value for the power and the computational cost associated with computing the $k$-th power of matrix $A$. Therefore, an alternative approach is recommended, which involves replacing the adjacency matrix $A$ with the weighted adjacency matrix $\tilde{A}$ proposed in GraphSNN \cite{GraphSNN}:

\begin{equation} \label{weighted matrix}
    \begin{aligned}
        & \tilde{A}_{v \mu}=\frac{\left|\mathbf{E}_{v u}\right|}{\left|\mathbf{V}_{v \mu}\right| \cdot\left|\mathbf{V}_{v \mu}-1\right|} \cdot\left|\mathbf{V}_{v \mu}\right|^\lambda \\
        & S_{v \mu}=\left(\mathbf{V}_{v \mu}, \mathbf{E}_{v \mu}\right)=S_v \cap S_\mu,(v, \mu) \in \mathbf{E}.
    \end{aligned}
\end{equation}

For a given graph $\mathcal{G}=(\mathbf{V}, \mathbf{E})$ with node set $\mathbf{V}$ and edge set $\mathbf{E}$, $S_v$ and $S_\mu$ are neighborhood subgraphs of nodes $v$ and $\mu$ respectively. $S_{v \mu}$ is overlap subgraph of $S_v$ and $S_\mu$, where $\mathbf{V}_{v \mu}\subseteq \mathbf{V}$ and $\mathbf{E}_{v \mu}\subseteq \mathbf{E}$. $\lambda$ is hyperparameter. Wijesinghe et al. \cite{GraphSNN} has demonstrated $\tilde{A}$ can help GNN learn the local structural information of a node to achieve powerful performance as higher-order Weisfeiler-Lehman (WL) test \cite{WL}, which can capture information beyond one-hop neighbors of a node and thus capture long-range inconsistency. 

We compare the above two approaches by evaluating the performance of numerical experiments in Sec.~\ref{Experiments}. Considering effectiveness, efficiency, and flexibility, we select $\tilde{A}$ as the reconstruction objective for MH-GAE.

\subsection{Candidate Group Sampling}
\label{sec-sample}
\begin{algorithm}[tb]
\caption{Candidate Groups Sampling}
\begin{algorithmic}[1]
    \STATE{\bfseries Input}: Input graph $\mathcal{G}=(\mathbf{V}, \mathbf{E})$, anchor node set $AN$, hyperparameter for tree depth $t$.
    \STATE{\bfseries Output}: Candidate group set $CG$
    
    \FOR{$v \in AN$}
        \FOR{$\mu \in AN$}
            \STATE $path_{v,\mu}$ = PathSearch($v, \mu, \mathbf{V}, \mathbf{E}$)
            \STATE add $path_{v,\mu}$ into $CG$
            \STATE $tree_{v,\mu}$ = TreeSearch($v, \mu, \mathbf{V}, \mathbf{E}, t$)
            \STATE add $tree_{v,\mu}$ into $CG$
        \ENDFOR
        \STATE $cycle_{v}$ = CycleSearch($v, \mathbf{V}, \mathbf{E}$)
        \STATE add $cycle_{v}$ into $CG$
    \ENDFOR

    \STATE return $CG$
\end{algorithmic}
\label{Candidate Groups Sampling}
\end{algorithm}

Starting from the anchor nodes identified by MH-GAE, we perform candidate group sampling to obtain potential anomaly groups for further distinguishing.

\subsubsection{Sampling algorithm}
In order to comprehensively capture potential anomaly groups in the graph, we employ three types of pattern search algorithms: path search, tree search, and cycle search, corresponding to Line 5, Line 7, and Line 10 in Alg.~\ref{Candidate Groups Sampling}, respectively. These three types of patterns (path, tree, and cycle) are fundamental building blocks of numerous complex patterns, making them essential for detecting various structures of anomaly groups in the graph. For instance, triangular and diamond patterns can be reduced to cycles, while the star pattern can be seen as a tree.


The path search algorithm is responsible for identifying linear sequences of nodes between two arbitrary anchor nodes $v$ and $\mu$, which can reveal the existence of simple and direct connections between nodes, such as sequential transactions or relationships. The tree search algorithm aims to detect hierarchical structures between anchor nodes $v$ and $\mu$, indicating the presence of branching relationships or nested dependencies in the graph. The cycle search algorithm finds closed loops or cycles which contain anchor node $v$, which may indicate recurring or cyclical patterns of interactions or transactions. Note the overlapping and repetitive patterns obtained through sampling contribute to increasing the number of candidate groups input to TPGCL, thereby enhancing the model's performance.

There are various implementations for pattern search. Considering efficiency, we choose Bellman-Ford algorithm \cite{pathsampling1,pathsampling2}, Breadth-First Search, and a proposed algorithm by Etienne et al.~\cite{cyclesampling} for path, tree, and cycle search respectively. 

\subsubsection{Time complexity}
Suppose the anchor node set $AN$ contains $m$ anchor nodes, due to the double traversal of anchor nodes in Lines 3-4 of Alg.~\ref{Candidate Groups Sampling}, $m^2$ times path searches and tree searches are performed, along with $m$ times cycle searches. The time complexity of both each path search and tree search is $O(|\mathbf{V}|+|\mathbf{E}|)$ \cite{pathsampling1,pathsampling2}, and each cycle search has a time complexity of $O((c+1)(|\mathbf{V}|+|\mathbf{E}|))$ \cite{cyclesampling}, where c is the number of cycles present in the graph. Therefore, the time complexity of candidate group sampling is $O(m(2m+c+1)(|\mathbf{V}|+|\mathbf{E}|))$. Since $c$ is often relatively small and can be neglected, the time complexity is $O(2m^2(|\mathbf{V}|+|\mathbf{E}|))$. Due to the rarity of anomaly groups, we select a small number of nodes as anchor nodes in the experiments, resulting in a small value for $m$. As a result, the overall time complexity is close to $O(|\mathbf{V}|+|\mathbf{E}|)$, which is acceptable for most datasets.

\subsection{Topology Pattern-based Graph Contrastive Learning}
\label{sec-gcl}




After identifying the candidate groups, we propose a novel variant of GCL, called Topology Pattern-based Graph Contrastive Learning (TPGCL). TPGCL incorporates specially crafted augmentations to capture topology pattern information, resulting in improved detection of anomalous groups. Compared to the common design of GCL, our approach differs in two key aspects: the way of generating augmented views and the construction of the training objective function.


\subsubsection{Topology pattern aware augmentations}

Our design relies on a critical assumption: criminal/anomaly groups often possess some specific topology patterns, which are commonly observed in various real-world scenarios. For example, the money laundering group tends to present a path-like structure \cite{flowscope, crysmurfing} and the Ponzi group usually has a tree-like structure~\cite{Ponzitree}.

The widely used augmentations, such as node/edge removing and feature masking, generate views by randomly perturbing the structure or node features of a graph. However, we point out these augmentations are not suitable for Gr-GAD because these augmentations do not consider the potential impact on the topological structure that is highly correlated with the labels. For example, node removal may delete the root node of a tree, which corresponds to the leader of a criminal gang in the real world. On the other hand, edge removal may delete a specific edge in a money laundering chain, causing the entire chain to break. The occurrence of such situations can potentially impact the performance of group detection, making it difficult to distinguish anomalous groups. 

Thus, to capture the topology pattern information, we introduce two types of topology augmentations:  Pattern Preserving Augmentation (\textbf{PPA}) and Pattern Breaking Augmentation (\textbf{PBA}).  Specifically, PPA will preserve the intrinsic topology patterns in a candidate group while PBA will break those patterns. We term the views generated by PPA (PBA) as positive (negative) views.  The key intuition is that the negative view will share less mutual information with an anomalous group by breaking the intrinsic topology patterns while the positive view will share more, which will facilitate our construction of the training objective function (detailed later).

Overall, PPA and PBA have two plausible features compared to previous widely used augmentations:

\begin{itemize}[leftmargin=*]
\item \textbf{Topology pattern oriented}. Unlike previous augmentations ignoring the possible impact on the local structures and topology patterns. PPA and PBA perturb groups from the perspective of a whole topology pattern. 
\item \textbf{Reduced randomness}. Previous augmentations usually make random perturbations, while PPA and PBA have prescribed impact (break or preserve) on the topology patterns. 
\end{itemize}





\begin{algorithm}[tb]
\caption{Topology Pattern Augmentations}
\begin{algorithmic}[1]
    \STATE{\bfseries Input}: Candidate group $g\in CG$
    \STATE{\bfseries Output}: Positive view $g_{p}$ and negative view $g_{n}$ of the input candidate group $g$.

    \STATE Initialize empty sets for tree, path and cycle: $Tree$, $Path$, $Cycle$. Initialize positive view $g_{p}$ and negative view $g_{n}$ as copies of the $g$.
    \STATE $Tree$, $Path$, $Cycle$ = TopologyPatternSearching($g$)
    \FOR{$T \in Tree$}
        \STATE locate the root node $r$ of $T$
        \STATE For $g_{n}$: Dropping the root node $r$.
        \STATE For $g_{p}$: Adding a new child $c$ to the root node $r$, the node attribute of new child is the average node attribute of other children. 
    \ENDFOR
    \FOR{$P \in Path$}
        \STATE locate the endpoint node $e$ of $P$ and the middle node $m$ of $P$.
        \STATE For $g_{n}$: Dropping the middle node $m$.
        \STATE For $g_{p}$: Adding a new neighbor node $n$ to the endpoint node $e$, the node attribute of new node is the average node attribute of other nodes in the path. 
    \ENDFOR
    \FOR{$C \in Cycle$}
        \STATE random chose two nodes $n_{1}, n_{2}$ of $C$.
        \STATE For $g_{n}$: Dropping $n_{1}, n_{2}$.
        \STATE For $g_{p}$: Adding a new node $n^{'}$ both link to $n_{1}, n_{2}$, the node attribute of new node is the average node attribute of other nodes in the cycle. 
    \ENDFOR

    \STATE return $g_{n}, g_{p}$
\end{algorithmic}
\label{Topology Pattern Contrastive Learning}
\end{algorithm}

To implement these two augmentations, topology pattern searching is first applied (Line 4 in Alg.~\ref{Topology Pattern Contrastive Learning}). For each candidate group, the topology pattern searching finds three patterns within this group: Tree, Path and Cycle. The reason for choosing these three patterns is that they represent the different basic classes of substructure, and other topology patterns such as triangles, diamonds, ad stars can be regarded as their different specific cases. The implementation of topology pattern searching is referred to Lines 5-10 of Alg.~\ref{Candidate Groups Sampling}.

After identifying three patterns in a candidate group, we break the patterns to generate the negative view and expand patterns in the positive view. The detailed process is shown in Lines 5-19 in Alg.~\ref{Topology Pattern Contrastive Learning}.
Specifically, for generating the negative view, we remove root nodes of trees (Line 7 in Alg.~\ref{Topology Pattern Contrastive Learning}), break paths by dropping the middle nodes of paths (Line 12 in Alg.~\ref{Topology Pattern Contrastive Learning}), and break cycles by dropping arbitrary two nodes (Line 17 in Alg.~\ref{Topology Pattern Contrastive Learning}). For the generation of the positive view, we expand and maintain intrinsic topology patterns simultaneously. That is, adding new child nodes to trees (Line 8 in Alg.~\ref{Topology Pattern Contrastive Learning}), prolonging paths by adding new nodes (Line 13 in Alg.~\ref{Topology Pattern Contrastive Learning}), and extending cycles by adding new nodes (Line 18 in Alg.~\ref{Topology Pattern Contrastive Learning}).

\subsubsection{Construction of learning objective}
Besides the augmentation process, another crucial ingredient in contrastive learning lies in the design of the contrastive loss function. In this section, we construct a specialized objective function that can explicitly take into account the topology pattern information.

Specifically, we begin by constructing a \textit{labeled version} of the contrastive loss (Eqn.~\eqref{target with label}) following the principle of Graph Information Bottleneck (GIB)~\cite{GIB}.  Since the labels are not available in the context of unsupervised learning, we eliminate the labels by leveraging the mutual information among the embeddings of positive and negative views, obtaining a \textit{label-free} objective (Eqn.~\eqref{eqn-label-free-obj}).  Finally, to make this objective easily computable, we further adopt a technique to efficiently estimate the mutual information, resulting in the final objective (Eqn.~\eqref{final loss function}) utilized in our method. \textit{To ensure better readability, we introduce the detailed construction below while deferring all the theoretic analysis in Section~\ref{Theoretical Supports}}.

Given a set of candidate groups $\mathcal{C} = \{c_{1}, c_{2}, \cdots, c_{m}\}$, we use the previous topology pattern aware augmentations (i.e., \textbf{PPA} and \textbf{PBA}) to generate
positive views $\mathcal{C}_p = \{c_{p1}, c_{p2}, \cdots, c_{pm}\}$ and negative views $\mathcal{C}_n = \{c_{n1}, c_{n2}, \cdots, c_{nm}\}$, respectively. 
Let $Y=\{0, 1\}^{m}$ be the set of labels corresponding to each group. $m$ is the number of candidate groups. 
Our goal is to train a model $f_\theta$ parameterized by $\theta$ that can effectively generate embeddings for groups in $\mathcal{C}$ without using the label information $Y$.

As indicated by \citeauthor{GIB}~\cite{GIB}, representation learning on graph-structured data should obey the principle of Graph Information Bottleneck (GIB): 
\begin{equation}\label{GIB}
    \max_\theta I(Y; f_\theta(\mathcal{C})) - I(\mathcal{C}; f_\theta(\mathcal{C})),
\end{equation}
where $I(\cdot ; \cdot)$ is mutual information, $f_\theta$ is a learnable model with parameters $\theta$ such as GNN, $\mathcal{C}$ and $Y$ are input graphs and corresponding labels. GIB aims at maximizing the mutual information between embeddings and labels while minimizing the mutual information between embeddings and inputs. 

Following this intention of GIB, we construct an objective function based on the positive and negative views, assuming that the labels $Y$ are available for now:
\begin{equation} \label{target with label}
    \max_\theta I(Y; f_\theta(\mathcal{C}_{p})) - I(Y; f_\theta(\mathcal{C}_{n})).
\end{equation}
Later, we provide a detailed theoretical analysis (Theorem.~\ref{theorem:1}) to demonstrate that maximizing the objective in Eqn.~\eqref{target with label} is actually equivalent to maximizing GIB.

Next, derive a label-free version of Eqn. \eqref{target with label} since the labels $Y$ are not available in the context of unsupervised learning. We use the following objective function to replace Eqn. \eqref{target with label}:
\begin{equation}\label{eqn-label-free-obj}
\min_\theta I(f_\theta(\mathcal{C}_{p}), f_\theta(\mathcal{C}_{n})). 
\end{equation}
The transition from Eqn. \eqref{target with label} to Eqn. \eqref{eqn-label-free-obj} relies on a crucial assumption: there exists a strong relevance between the topology patterns and labels. Essentially, this assumption allows us to replace the mutual information between labels and embeddings with mutual information among the embeddings of the positive and negative views. We provide Theorem.~\ref{theorem:2} to support this claim.



Still, the objective function~\eqref{eqn-label-free-obj} cannot be used in training since it is hard to compute the mutual information directly. To address this, we adopt the techniques from MINE~\cite{MINE} to implement the objective. Specifically, function  \eqref{eqn-label-free-obj} can be rewritten as:

\begin{equation}\label{final loss function}
    \begin{aligned}
        \mathcal{L}=\min_{f_\theta, \Phi}& -\frac{1}{m} \sum_{i=1}^m \Phi\left(f_\theta(\mathcal{C}_{pi}), f_\theta(\mathcal{C}_{ni})\right) \\
        & +\log \frac{1}{m} \sum_{i=1}^m \sum_{j=1, j \neq i}^m e^{\Phi}\left(f_\theta(\mathcal{C}_{pi}), f_\theta(\mathcal{C}_{nj})\right).
    \end{aligned}
\end{equation}
In the above function, $\Phi$ is a trainable estimator which is usually implemented through MLP. After minimizing Eqn.~\eqref{final loss function} and convergence, the TPGCL outputs the input graph's embedding which contains rich topology pattern and label-related information. 
Finally, the output embeddings can be input to outlier detection or other unsupervised classifiers to be classified.

\section{Theoretical Analysis}
\label{Theoretical Supports}

In this section, we provide the theoretical foundation for the construction of objective function from the perspective of mutual information (MI) and GIB. 

Firstly, for the given input candidate group set $\mathcal{C}$ and corresponding view sets $\mathcal{C}_n$, $\mathcal{C}_p$, we provide the following assumption.


\begin{assumption}\label{assumption:1}
There exists a correlation between the topology patterns and label information. That is, anomalous groups are more likely to exhibit specific topology patterns.
\end{assumption}

It is important to note that the above assumption is rooted in extensive previous research works \cite{flowscope, smurfAML, autoaudit, crysmurfing, antibenford}, where the prior knowledge of topology patterns has been successfully employed to identify anomalies. Moreover, we emphasize that Assumption~\ref{assumption:1} allows us to provide a theoretical justification for our design of learning objective; it does not guarantee the absence of counterexamples in practice. Nevertheless, if real-world anomalous groups \textit{largely} conform to this assumption, we can achieve a good learning result (which is indeed observed in our evaluation).


Under the assumption of its validity, we can propose the following proposition.

\begin{proposition}\label{proposition:1}
    Mutual information $I(Y; f_\theta(\mathcal{C}))$ is approximately equal to entropy $H(Y)$.
\end{proposition}

\begin{proof}

For anomaly detection, the label set $Y$ can be written as $Y=\{0, 1\}$, where $0$ and $1$ represent normal and anomalous data. Suppose most of the anomaly groups tend to form topology pattern $i$ whose embedding is $h_i$ and the probability of anomaly groups forming other topology patterns is $\xi$:

\begin{equation}
    \left\{
    \begin{array}{l}
        P\left(Y=1 \mid f_\theta(\mathcal{C})=h_i \right) = \frac{P(Y=1) - \xi}{P(f_\theta(\mathcal{C})=h_i)}, \\ P\left(Y=0 \mid f_\theta(\mathcal{C})=h_i\right) = \frac{P(Y=0, f_\theta(\mathcal{C}) = h_i)}{P(f_\theta(\mathcal{C})=h_i)}, \\
        P\left(Y=1 \mid f_\theta(\mathcal{C}) \neq h_i \right) = \frac{\xi}{P(f_\theta(\mathcal{C}) \neq h_i)}, \\ P\left(Y=0 \mid f_\theta(\mathcal{C}) \neq h_i \right) = \frac{P(Y=0, f_\theta(\mathcal{C}) \neq h_i)}{P(f_\theta(\mathcal{C})\neq h_i)}.
    \end{array}
    \right.
\end{equation}
We assume a strong correlation between anomalies and topology patterns $i$, with few anomaly groups forming other topology patterns, which means $\xi$ is small, resulting in:  

\begin{equation}
    \left\{
    \begin{array}{l}
        P\left(Y=1 \mid f_\theta(\mathcal{C})=h_i \right) \rightarrow \frac{P(Y=1)}{P(f_\theta(\mathcal{C})=h_i)}, 
        \\ P\left(Y=0 \mid f_\theta(\mathcal{C})=h_i\right) \rightarrow 1-\frac{P(Y=1)}{P(f_\theta(\mathcal{C})=h_i)}, \\
        P\left(Y=1 \mid f_\theta(\mathcal{C}) \neq h_i \right) \rightarrow 0, 
        \\ P\left(Y=0 \mid f_\theta(\mathcal{C}) \neq h_i \right) \rightarrow 1.
    \end{array}
    \right. \label{conditional prob}
\end{equation}
Furthermore, with few samples being anomaly samples (always less than 0.1) and the vast diversity of topology patterns present in the graph, the probability of a certain topology pattern is small. As a result, both $P(Y=1)$ and $P(f_\theta(\mathcal{C})=h_i)$ are small too. Therefore, the value of the conditional entropy $ H(Y|f_\theta(\mathcal{C}))$ tends to be zero, or a tiny small value can be neglected, resulting in $I(Y;f_\theta(\mathcal{C}))$ approximately equals to $H(Y)$, since the definition of MI:

\begin{equation}
    I(Y;f_\theta(\mathcal{C})) = H(Y) - H(Y|f_\theta(\mathcal{C})).
    \label{MI definition}
\end{equation}
\end{proof}

Proposition~\ref{proposition:1} indicates that under the Assumption~\ref{assumption:1}, the conditional entropy is significantly small, making it negligible. In this context, the mutual information can be approximated as being equal to the information entropy of the labels. 
Note the conditional entropy $H(Y|f_\theta(\mathcal{C}))$ tends to be zero does not mean that $H(f_\theta(\mathcal{C})|Y)$ approximates to be zero. That is, anomaly groups tend to form certain topology patterns, but it does not mean groups with certain topology patterns must be anomaly groups.

Based on Proposition~\ref{proposition:1}, we provide two Lemmas:

\begin{lemma}\label{lemma:1}
The mutual information of labels $Y$ and embedding of negative view $f_\theta(\mathcal{C}_{n})$ is less than the mutual information $I(Y; f_\theta(\mathcal{C}))$: $I(Y; f_\theta(\mathcal{C}_{n})) \leq I(Y; f_\theta(\mathcal{C}))$.
\end{lemma}

\begin{proof}

Since the topology patterns of anomaly groups have been broken and some anomaly groups' topology patterns have converted to other patterns, the conditional probabilities $P(Y=1|f_\theta(\mathcal{C}_n)=h_i)$ will decrease and $P(Y=1|f_\theta(\mathcal{C}_n)\neq h_i)$ will increase. In other words, the distribution of conditional probabilities becomes less extreme, and thus conditional entropy becomes larger:

\begin{equation}
    H(Y \mid f_\theta(\mathcal{C}_n)) \geq  H(Y \mid f_\theta(\mathcal{C})),
\end{equation}
and the mutual information between labels and negative views fulfills:

\begin{equation}
    \begin{aligned}
        I(Y;f_\theta(\mathcal{C}_n)) & = H(Y) - H(Y \mid f_\theta(\mathcal{C}_n)) \\
        & \leq H(Y) - H(Y \mid f_\theta(\mathcal{C})) \\
        & = I(Y;f_\theta(\mathcal{C})).
    \end{aligned}
\end{equation}
\end{proof}

\begin{lemma}\label{lemma:2}
The mutual information of labels $Y$ and embedding of positive view $f_\theta(\mathcal{C}_{p})$ is equal to the mutual information $I(Y; f_\theta(\mathcal{C}))$: $I(Y; f_\theta(\mathcal{C}_{p})) = I(Y; f_\theta(\mathcal{C}))$.
\end{lemma}

\begin{proof}

Topology patterns keep consistent because PPA augmentation maintains the topology patterns. Due to the labels being fixed, the conditional probabilities fulfill:

$$
    \begin{cases}
        P\left(Y=1 \mid f_\theta(\mathcal{C}_p)=h_i\right)= P\left(Y=1 \mid f_\theta(\mathcal{C})=h_i\right) \\
        P\left(Y=0 \mid f_\theta(\mathcal{C}_p)=h_i\right)= P\left(Y=0 \mid f_\theta(\mathcal{C})=h_i\right) \\
        P\left(Y=1 \mid f_\theta(\mathcal{C}_p)\neq h_i\right)= P\left(Y=1 \mid f_\theta(\mathcal{C})\neq h_i\right) \\
        P\left(Y=0 \mid f_\theta(\mathcal{C}_p)\neq h_i\right)= P\left(Y=0 \mid f_\theta(\mathcal{C})\neq h_i\right).
    \end{cases}
$$
Naturally, the conditional entropy $H(Y \mid f_\theta(\mathcal{C}_p))$ and $I(Y;f_\theta(\mathcal{C}_p))$ fulfill:
\begin{equation}
    \begin{aligned}
    H(Y \mid f_\theta(\mathcal{C}_p)) = & H(Y \mid f_\theta(\mathcal{C})),\\ I(Y;f_\theta(\mathcal{C}_p)) = & I(Y;f_\theta(\mathcal{C})).
    \end{aligned}
\end{equation}

\end{proof}

Lemma.~\ref{lemma:1} indicates the $I(Y; f_\theta(\mathcal{C}_{n}))$ between labels $Y$ and the embedding of PBA outputs negative view $\mathcal{C}_{n}$ decreases after PBA breaks the intrinsic patterns in a graph since the specific kinds of anomalous groups have vanished and $f_\theta(\mathcal{C}_{n})$ is not sufficient to determine $Y$ any more. In other words, PBA breaks the relevance between topology patterns and ground truth. Unlike PBA breaking intrinsic patterns, PPA maintains and expands intrinsic patterns and therefore saves the most relevant information. 

With Lemma.~\ref{lemma:1} and Lemma.~\ref{lemma:2} in hand, we provide the theoretical supports of maximizing objective function~\eqref{target with label} is equal to maximizing GIB:

\begin{theorem}\label{theorem:1}
 GIB and $I(Y; f_\theta(\mathcal{C}_{p})) - I(Y; f_\theta(\mathcal{C}_{n}))$ share a common maximization objective.
\end{theorem}

\begin{proof}

According to the properties of the MI, we have $I(\mathcal{C}; f_\theta(\mathcal{C})) \geq 0$ and thus:

\begin{equation}
    I(Y; f_\theta(\mathcal{C})) - I(\mathcal{C}; f_\theta(\mathcal{C})) \leq I(Y; f_\theta(\mathcal{C})),
\end{equation}
which shows $I(Y; f_\theta(\mathcal{C}))$ is the upper bound of GIB. Therefore, maximizing GIB is actually maximizing $I(Y; f_\theta(\mathcal{C}))$. That is, maximizing the mutual information between labels and embeddings of the input graphs. 

Based on Lemma 1, Lemma 2, and the non-negativity property of MI, the following equation holds:

\begin{equation}
    \begin{aligned}
        I(Y; f_\theta(\mathcal{C})) & \geq I(Y; f_\theta(\mathcal{C})) - I(Y; f_\theta(\mathcal{C}_{n})) \\
        & = I(Y; f_\theta(\mathcal{C}_{p})) - I(Y; f_\theta(\mathcal{C}_{n})),
    \end{aligned}
\end{equation}
which shows $I(Y; f_\theta(\mathcal{C}_{p})) - I(Y; f_\theta(\mathcal{C}_{n}))$ is the lower bound of $I(Y; f_\theta(\mathcal{C}))$. Thus GIB and Eqn.~\eqref{target with label} share a common maximization objective, which means maximizing $I(Y; f_\theta(\mathcal{C}_{p})) - I(Y; f_\theta(\mathcal{C}_{n}))$ is equal to maximizing $I(Y; f_\theta(\mathcal{C}))$.

\end{proof}

As function~\eqref{target with label} still relies on labels, we introduce Theorem~\ref{theorem:2} to convert function~\eqref{target with label} to unsupervised  function~\eqref{eqn-label-free-obj}:

\begin{theorem}\label{theorem:2}
 Maximizing $I(Y; f_\theta(\mathcal{C}_{p})) - I(Y; f_\theta(\mathcal{C}_{n}))$ is equal to minimizing $I(f_\theta(\mathcal{C}_{p}), f_\theta(\mathcal{C}_{n}))$.
\end{theorem}

\begin{proof}

Considering the MI between the joint distribution of ($f_\theta(\mathcal{C}_{p})$, $Y$), and the distribution of $f_\theta(\mathcal{C}_{n})$: $I((f_\theta(\mathcal{C}_{p}), Y); f_\theta(\mathcal{C}_{n}))$. Due to the properties of MI, we have:

\begin{equation}\label{joint MI}
    \begin{aligned}
        & I((f_\theta(\mathcal{C}_{p}), Y); f_\theta(\mathcal{C}_{n})) \\
        & = I(f_\theta(\mathcal{C}_{p}); f_\theta(\mathcal{C}_{n})) + I(Y; f_\theta(\mathcal{C}_{n}) | f_\theta(\mathcal{C}_{p})) \\
        & = I(Y; f_\theta(\mathcal{C}_{n})) + I(f_\theta(\mathcal{C}_p);f_\theta(\mathcal{C}_{n}) | Y),
    \end{aligned}
\end{equation}
and conditional entropy $H(Y|f_\theta(\mathcal{C}_{p}))$ can be represented as:
$$
    H(Y|f_\theta(\mathcal{C}_{p})) = I(Y;f_\theta(\mathcal{C}_{n})|f_\theta(\mathcal{C}_{p})) + H(Y|f_\theta(\mathcal{C}_{p});f_\theta(\mathcal{C}_{n})).
$$
Based on Proposition.~\ref{proposition:1} and Lemma.~\ref{lemma:2}, $H(Y|f_\theta(\mathcal{C})) = H(Y|f_\theta(\mathcal{C}_{p}))\rightarrow 0$, thus:
\begin{equation}\label{tend to zero}
    \begin{aligned}
        & I(Y;f_\theta(\mathcal{C}_{n})|f_\theta(\mathcal{C}_{p})) \rightarrow 0, 
        & H(Y|f_\theta(\mathcal{C}_{p});f_\theta(\mathcal{C}_{n})) \rightarrow 0.
    \end{aligned}
\end{equation}
According to the Eqn.~\eqref{joint MI} and Eqn.~\eqref{tend to zero}, we have:
\begin{equation}
    I(f_\theta(\mathcal{C}_{p}); f_\theta(\mathcal{C}_{n})) = I(Y; f(\mathcal{C}_{n})) + I(f_\theta(\mathcal{C}_p);f_\theta(\mathcal{C}_{n}) | Y).
\end{equation}
According to the properties of MI, $I(f_\theta(\mathcal{C}_n);f_\theta(\mathcal{C}_{p}) | Y)\geq 0$, thus we have:
\begin{equation} \label{inequality}
    I(f_\theta(\mathcal{C}_{p}); f_\theta(\mathcal{C}_{n})) \geq I(Y; f_\theta(\mathcal{C}_{n})).
\end{equation}
Furthermore, according to Proposition.~\ref{proposition:1}, we have:
\begin{equation} \label{fixed}
    I(Y; f_\theta(\mathcal{C}_{p})) \approx H(Y).
\end{equation}
Based on Eqn.~\eqref{inequality} and Eqn.~\eqref{fixed}, the ideal objective function Eqn.~\eqref{target with label} can be rewritten as:
\begin{equation}
    \begin{aligned}
        I(Y; f_\theta(\mathcal{C}_{p})) - I(Y; f_\theta(\mathcal{C}_{n})) &
        \approx H(Y) - I(Y; f_\theta(\mathcal{C}_{n})) \\
        \geq & H(Y) - I(f_\theta(\mathcal{C}_{p}); f_\theta(\mathcal{C}_{n})).
    \end{aligned}
\end{equation}
It demonstrates the $H(Y) - I(f_\theta(\mathcal{C}_{p}); f_\theta(\mathcal{C}_{n}))$ is the lower bound of $I(Y; f_\theta(\mathcal{C}_{p})) - I(Y; f_\theta(\mathcal{C}_{n}))$. Therefore, maximizing $H(Y) - I(f_\theta(C_{p}); f_\theta(\mathcal{C}_{n}))$ is equal to maximizing $I(Y; f_\theta(\mathcal{C}_{p})) - I(Y; f_\theta(\mathcal{C}_{n}))$. The objective function is:
\begin{equation}
    \begin{aligned}
        & \max_{f_\theta} H(Y) - I(f_\theta(\mathcal{C}_{p}); f_\theta(\mathcal{C}_{n})) \\
        = & \min_{f_\theta} I(f_\theta(\mathcal{C}_{p}); f_\theta(\mathcal{C}_{n})) - H(Y) \\
        = & \min_{f_\theta} I(f_\theta(\mathcal{C}_{p}); f_\theta(\mathcal{C}_{n})).
    \end{aligned}
\end{equation}
The above function holds due to the $H(Y)$ being a fixed value just determined by the distribution of label set $Y$ and irrelevant to model $f_\theta$.
\end{proof}

Theorem.~\ref{theorem:2} indicates we can train TPGCL in an unsupervised manner since the positive view contains more label-relevance information by maintaining intrinsic topology patterns and the negative view's information is label-irrelevance after breaking intrinsic topology patterns.

\section{Experiments}
\label{Experiments}
In this section, we evaluate our proposed framework and aim to answer the following research question (RQ):
\begin{enumerate}
\item \textbf{RQ1:} Is \textbf{TP-GrGAD} effective for \textbf{Gr-GAD}?
\item \textbf{RQ2:} How does \textbf{TP-GrGAD} perform comparing to existing N-GAD and Sub-GAD methods?
\item \textbf{RQ3:} To what extent do the newly proposed MH-GAE and TPGCL contribute to our framework?
\end{enumerate} 
We begin by introducing datasets, evaluation metrics, and experimental setup. Subsequently, we present a performance comparison to address \textbf{RQ1} and \textbf{RQ2} in Section \ref{Exp: performance comparison}. \textbf{RQ3} is addressed through ablation studies in Section \ref{Exp: ablation studies}. Furthermore, to comprehend the effectiveness of the proposed method, we provide visualizations in Section \ref{Exp:visualization}.

\subsection{Experiment Setup}

\subsubsection{Datesets}
The datasets we employ include three synthetic datasets (\textbf{simML}~\cite{AMLSim}, \textbf{Cora-group}, and \textbf{CiteSeer-group}) and two real-world datasets (\textbf{AMLPublic}~\cite{AMLPublic} and \textbf{Ethereum-TSGN}~\cite{EthereumTSGN}). To ensure the above datasets can be used to evaluate the performance of Gr-GAD, we first locate and label anomaly groups according to the labeled anomaly nodes or inject anomaly groups. Detailed statistics of datasets are presented in Table~\ref{datasets}, where \#Attr and Avg. size are dimensions of node attributes and average anomaly group size, respectively. 


\begin{table}[th]
\centering
\small
\caption{\centering{Statistical details of the datasets.}}
\resizebox{0.8\columnwidth}{!}{%
\begin{tabu}{c|c|c|c|c|c}\tabucline[1.5pt]{-}
Dataset & \#Node & \#Edge & \#Attr & \makecell[c]{\#Anomaly\\ Group} & \makecell[c]{Avg.\\ size} \\ \hline
simML & 2,768 & 4,226 & 3,123 & 74 &3.52\\ \hline
Cora-g & 2,847 & 10,792 & 1,433 & 22 &6.32\\ \hline
CiteSeer-g & 3,463 & 9,334 & 3,703 & 22 &6.18\\ \hline
AMLP & 16,720 & 17,238 & 16 & 19 &19.05\\ \hline
Eth & 1,823 & 3,254 & 13 & 17 &7.23 \\\tabucline[1.5pt]{-}
\end{tabu}}\label{datasets}
\end{table}

\begin{itemize}[leftmargin=*]
\item \textbf{AMLPublic}: This dataset contains 90,000 bank accounts. After data cleaning, we construct a graph comprising 16,720 nodes and 17,238 edges, where nodes represent bank accounts and edges represent transactions. By leveraging provided labels, we identified 19 abnormal groups involve in money laundering.  We abbreviate it as \textbf{AMLP} in this paper.

\item \textbf{Ethereum-TSGN}: This dataset was collected by Wang et al.~\cite{EthereumTSGN} from Ethereum, where nodes and edges represent user accounts and transactions, respectively. It consists of 17 distinct phishing groups involved in Phishing scams. We abbreviate it as \textbf{Eth} in this paper. 




\item \textbf{simML}: A synthetic money laundering dataset opened in Kaggle, which was generated by IBM AMLSim \cite{AMLSim} based on financial principles. This dataset contains over 2,000 nodes and 4,000 edges where nodes represent accounts and edges represent transactions. 

\item \textbf{Cora-group}: A synthetic Gr-GAD dataset based on a widely used node classification dataset: Cora \cite{Cora}. We choose anchor nodes and add new nodes to link these anchor nodes to form anomaly groups. Attributes of new adding nodes are generated by adding Gaussian noise to the anchor nodes' attributes. This paper abbreviates it as \textbf{Cora-g}.

\item \textbf{CiteSeer-group}: Generated from CiteSeer \cite{CiteSeer} dataset in the same way as Cora-group, which contains bag-of-words representation of documents and citation links between the documents. This paper abbreviates it as \textbf{CiteSeer-g}.

\end{itemize}

To validate our assumption that anomaly groups tend to exhibit specific topology patterns based on their functions, we evaluate the topology patterns of anomaly groups using two real-world datasets. As presented in Table~\ref{topology patterns statistic}, although a few counterexamples exist, the majority of anomaly groups in AMLPublic exhibit path-like topology patterns, while those in Ethereum-TSGN display tree-like and cycle-like structures. This observation confirms our assumption and provides support for the theoretical analysis in Section~\ref{Theoretical Supports}. To illustrate this further, we select two representative anomaly groups from AMLPublic and Ethereum-TSGN, showcased in Fig.~\ref{example anomaly group}. The example anomaly group in AMLPublic forms a coherent path, reflecting a real-world money laundering flow. Conversely, Ethereum-TSGN's example anomaly group is related to phishing activities and comprises a cycle with an inner cycle, indicating significant interactions within the group.

\begin{table}[!h]
\centering
\small
\caption{\centering{Topology pattern statistic.}}
\begin{tabu}{c|c|c|c|c}\tabucline[1.5pt]{-}
     & \#Path & \#Tree & \#Cycle & \#Total \\ \hline
AMLPublic & 18   & 1    & 0     & 19    \\ \hline
Ethereum-TSGN  & 1    & 9    & 7     & 17   \\\tabucline[1.5pt]{-}
\end{tabu}\label{topology patterns statistic}
\end{table}

\begin{figure}[!ht]
\centering
\begin{subfigure}{0.23\textwidth}
    \centering
    \includegraphics[width=2.5cm, height=1.5cm]{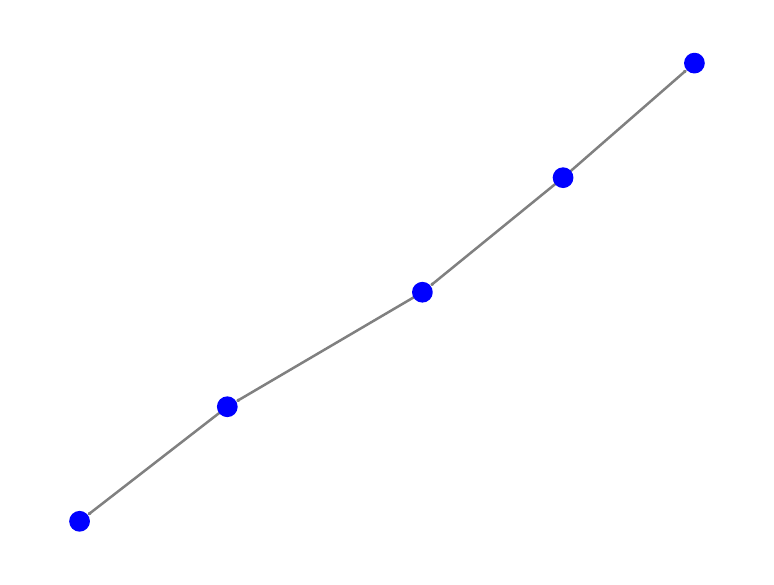}
    \caption{}
\end{subfigure}\hspace{2mm}
\begin{subfigure}{0.23\textwidth}
    \centering
    \includegraphics[width=2.5cm, height=1.5cm]{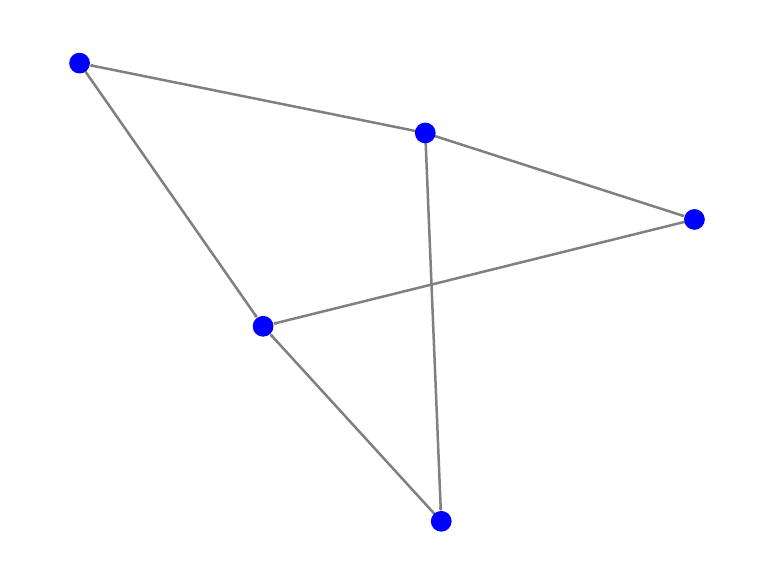}
    \caption{}
\end{subfigure}
\caption{Anomaly group examples of (a) AMLPublic and (b) Ethereum-TSGN.}
\label{example anomaly group}
\end{figure}

\begin{table*}[!h]
\centering
\caption{\centering{Results on all datasets($\pm$ standard error). The best results are highlighted in bold.}}
\small
\resizebox{1.8\columnwidth}{!}{%
\begin{tabular}{cc|c|ccc|cc|c}
\hline
\multicolumn{2}{c|}{\multirow{2}{*}{Datasets}} & \multirow{2}{*}{Metric} & \multicolumn{3}{c|}{N-GAD} & \multicolumn{2}{c|}{Sub-GAD} & Ours \\ \cline{4-9}
\multicolumn{2}{c|}{} & & \multicolumn{1}{c|}{DOMINANT} & \multicolumn{1}{c|}{DeepAE} & ComGA & \multicolumn{1}{c|}{DeepFD} & AS-GAE & TP-GrGAD \\ \hline
\multicolumn{1}{c|}{\multirow{6}{*}{Real-world}} & \multirow{3}{*}{Ethereum-TSGN} & CR & \multicolumn{1}{c|}{0.19±0.01} & \multicolumn{1}{c|}{0.19±0.} & 0.18±0.01 & \multicolumn{1}{c|}{0.27±0.04} & 0.39±0.02 & \textbf{0.81±0.10} \\
\multicolumn{1}{c|}{} & & F1 & \multicolumn{1}{c|}{0.14±0.01} & \multicolumn{1}{c|}{0.14±0.} & 0.14±0.01 & \multicolumn{1}{c|}{0.50±0.09} & 0.44±0.03 & \textbf{0.73±0.04} \\
\multicolumn{1}{c|}{} & & AUC & \multicolumn{1}{c|}{0.51±0.01} & \multicolumn{1}{c|}{0.48±0.} & 0.49±0.01 & \multicolumn{1}{c|}{0.52±0.03} & 0.49±0.03 & \textbf{0.86±0.06} \\ \cline{2-9}
\multicolumn{1}{c|}{} & \multirow{3}{*}{AMLPublic} & CR & \multicolumn{1}{c|}{0.10±0.01} & \multicolumn{1}{c|}{0.10±0.01} & 0.10±0.01 & \multicolumn{1}{c|}{0.32±0.01} & {0.28±0.09} & \textbf{0.89±0.} \\
\multicolumn{1}{c|}{} & & F1 & \multicolumn{1}{c|}{0.37±0.03} & \multicolumn{1}{c|}{0.37±0.06} & 0.38±0.03 & \multicolumn{1}{c|}{0.55±0.08} & 0.39±0.03 & \textbf{0.90±0.} \\
\multicolumn{1}{c|}{} & & AUC & \multicolumn{1}{c|}{0.89±0.} & \multicolumn{1}{c|}{0.87±0.04} & 0.83±0.08 & \multicolumn{1}{c|}{0.56±0.13} & {0.70±0.06} & \textbf{0.85±0.04} \\ \hline
\multicolumn{1}{c|}{\multirow{9}{*}{Synthetic}} & \multirow{3}{*}{simML} & CR & \multicolumn{1}{c|}{0.33±0.01} & \multicolumn{1}{c|}{0.27±0.03} & 0.31±0.04 & \multicolumn{1}{c|}{0.12±0.02} & 0.29±0.01 & \textbf{0.84±0.04} \\
\multicolumn{1}{c|}{} & & F1 & \multicolumn{1}{c|}{0.45±0.02} & \multicolumn{1}{c|}{0.54±0.03} & 0.54±0.06 & \multicolumn{1}{c|}{0.37±0.05} & 0.53±0.04 & \textbf{0.76±0.04} \\
\multicolumn{1}{c|}{} & & AUC & \multicolumn{1}{c|}{0.50±0.04} & \multicolumn{1}{c|}{0.49±0.05} & 0.51±0.05 & \multicolumn{1}{c|}{0.64±0.06} & 0.56±0.06 & \textbf{0.84±0.01} \\ \cline{2-9}
\multicolumn{1}{c|}{} & \multirow{3}{*}{Cora-group} & CR & \multicolumn{1}{c|}{0.13±0.06} & \multicolumn{1}{c|}{0.13±0.04} & 0.13±0.05 & \multicolumn{1}{c|}{0.18±0.01} & 0.16±0.1 & \textbf{0.93±0.2} \\
\multicolumn{1}{c|}{} & & F1 & \multicolumn{1}{c|}{0.43±0.08} & \multicolumn{1}{c|}{0.41±0.03} & 0.42±0.04 & \multicolumn{1}{c|}{0.46±0.07} & 0.60±0.06 & \textbf{0.75±0.02} \\
\multicolumn{1}{c|}{} & & AUC & \multicolumn{1}{c|}{0.53±0.02} & \multicolumn{1}{c|}{0.52±0.02} & 0.53±0.02 & \multicolumn{1}{c|}{0.52±0.05} & 0.55±0.02 & \textbf{0.73±0.02} \\ \cline{2-9}
\multicolumn{1}{c|}{} & \multirow{3}{*}{CiteSeer-group} & CR & \multicolumn{1}{c|}{0.14±0.03} & \multicolumn{1}{c|}{0.14±0.02} & 0.14±0.03 & \multicolumn{1}{c|}{0.26±0.03} & 0.20±0.09 & \textbf{0.72±0.05} \\
\multicolumn{1}{c|}{} & & F1 & \multicolumn{1}{c|}{0.35±0.01} & \multicolumn{1}{c|}{0.34±0.01} & 0.35±0.01 & \multicolumn{1}{c|}{0.52±0.05} & 0.54±0.04 & \textbf{0.85±0.01} \\
\multicolumn{1}{c|}{} & & AUC & \multicolumn{1}{c|}{0.72±0.02} & \multicolumn{1}{c|}{0.74±0.01} & 0.72±0.02 & \multicolumn{1}{c|}{0.52±0.04} & 0.59±0.02 & \textbf{0.87±0.03} \\ \hline
\end{tabular}}\label{experiment results}
\vspace{-0.4cm}
\end{table*}

\subsubsection{Evaluation metrics} 
Unlike N-GAD and Sub-GAD evaluate the performance of classification from the perspective of node level, i.e., how many nodes are classified correctly, we argue the measurements of Gr-GAD shall focus on group-level and cover two folds: \textbf{Detection Accuracy}, i.e., how many groups are classified correctly, and \textbf{Detection Completeness}, i.e., how complete are the detected groups.

For detection accuracy, we use two widely used binary classification evaluation metrics: F1-score and AUC. Note the calculation of these two metrics is group-wise basis. 

Existing studies have neglected detection completeness and lack appropriate evaluation metrics for it. Therefore, we propose a new metric Completeness Ratio (CR) to measure detection completeness. Specifically, given the ground truth anomaly group set $\mathcal{C}$ and predicted anomaly group set $\hat{\mathcal{C}}$, for a ground truth group $c_g=(\mathbf{V}_g, \mathbf{E}_g), c_g\in \mathcal{C}$, its completeness score $s_g$ is:
\begin{equation}\label{completeness score}
    s_{g} = \max_{\hat{c_i}\in\hat{\mathcal{C}}}\frac{1}{2} \cdot (\frac{|\hat{\mathbf{V}_{i}}|\cap|\mathbf{V}_{g}|}{|\mathbf{V}_{g}|} + \frac{|\hat{\mathbf{V}_{i}}|\cap|\mathbf{V}_{g}|}{|\hat{\mathbf{V}_{i}}|}), \hat{c_i}=(\hat{\mathbf{V}_i}, \hat{\mathbf{E}_i}).
\end{equation}
The completeness score $s_g$ is composed of two parts. The first part (the term on the right-hand side of the addition in Eqn.~\eqref{completeness score}) measures the proportion of overlapping nodes between $\hat{c}_i$ and $c_g$, relative to the total number of nodes in $c_g$. In other words, it quantifies what fraction of $c_g$ has been detected or whether $c_g$ has been completely detected. The second part (the term on the left-hand side of the addition in Eqn.~\eqref{completeness score}) measures the proportion of overlapping nodes between $\hat{c}_i$ and $c_g$, relative to the total number of nodes in $\hat{c}_i$. It assesses redundant nodes contained in $\hat{c}_i$.

Then, the CR value is the average completeness scores of all ground truth groups:
\begin{equation}
    CR = \frac{1}{|\mathcal{C}|} \sum_{c_{g}\in \mathcal{C}} s_{g}.
\end{equation}
CR measures the number of missing and redundant nodes in the predicted groups simultaneously. The closer the CR value is to 1, the more complete the prediction group is. 

\subsubsection{Baselines}
The baselines used for comparison include state-of-the-art unsupervised N-GAD and Sub-GAD methods. 
\begin{itemize}[leftmargin=*]
    \item Node-level anomaly detection (N-GAD): DOMINANT~\cite{dominant}, ONE~\cite{ONE} and ComGA \cite{ComGA}, which are based on GAE and reconstruction error assumption to detect anomaly nodes in an unsupervised manner. We generalize them to do the task of Gr-GAD by following the style of AS-GAE \cite{AS-GAE}, sampling groups from their detected anomaly nodes via connected component detection.
    \item Subgraph-level anomaly detection (Sub-GAD): DeepFD \cite{DeepFD} and AS-GAE \cite{AS-GAE}, which locate anomaly nodes in N-GAD style first and extract anomaly subgraphs from the anomaly node set by clustering or connected component detection.
\end{itemize}
\subsubsection{Experiment environment and setup}
For a fair comparison, each method is run on a Linux system with 16 Gen Intel(R) Core(TM) i9-12900F cores and an NVIDIA GeForce RTX 3090. The codes and the parameters used for the comparison are available from the authors' public link or widely-used public implementations~\cite{dominantcode, DeepAEcode, ComGAcode, DeepFDcode, ASGAEcode}. In the experiments, we typically select the top 10\% nodes with the highest reconstruction errors as anchor nodes and sample candidate groups from them. We select 2-layer GCN \cite{GCN} as the encoder of MH-GAE and TPGCL. The TPGCL's output embeddings are 64 dimensional and we employ the state-of-the-art outlier detector named ECOD \cite{ECOD} for our experimental evaluations. More details at our github repository.

\subsection{Performance Comparison}
\label{Exp: performance comparison}

As shown in Table.~\ref{experiment results}, the proposed method in this study exhibits significant superiority over the N-GAD and Sub-GAD methods across all metrics, particularly in terms of the CR metric. Whether applied to real-world or synthetic datasets, both N-GAD and Sub-GAD methods consistently exhibit considerably low CR values, ranging from approximately 0.1 to 0.4. This leads to the ineffective classification of anomalous groups by the N-GAD and Sub-GAD methods, resulting in lower F1 scores and AUC values.

As observed in Sec.~\ref{MH-GAE section}, N-GAD and Sub-GAD methods tend to identify isolated nodes and smaller-sized groups due to their limited capability in capturing long-range inconsistencies, as depicted in Fig.~\ref{fig:group_size}.

\begin{figure}[!ht]
    \centering
    \includegraphics[width=0.5\textwidth, height=0.18\textheight]{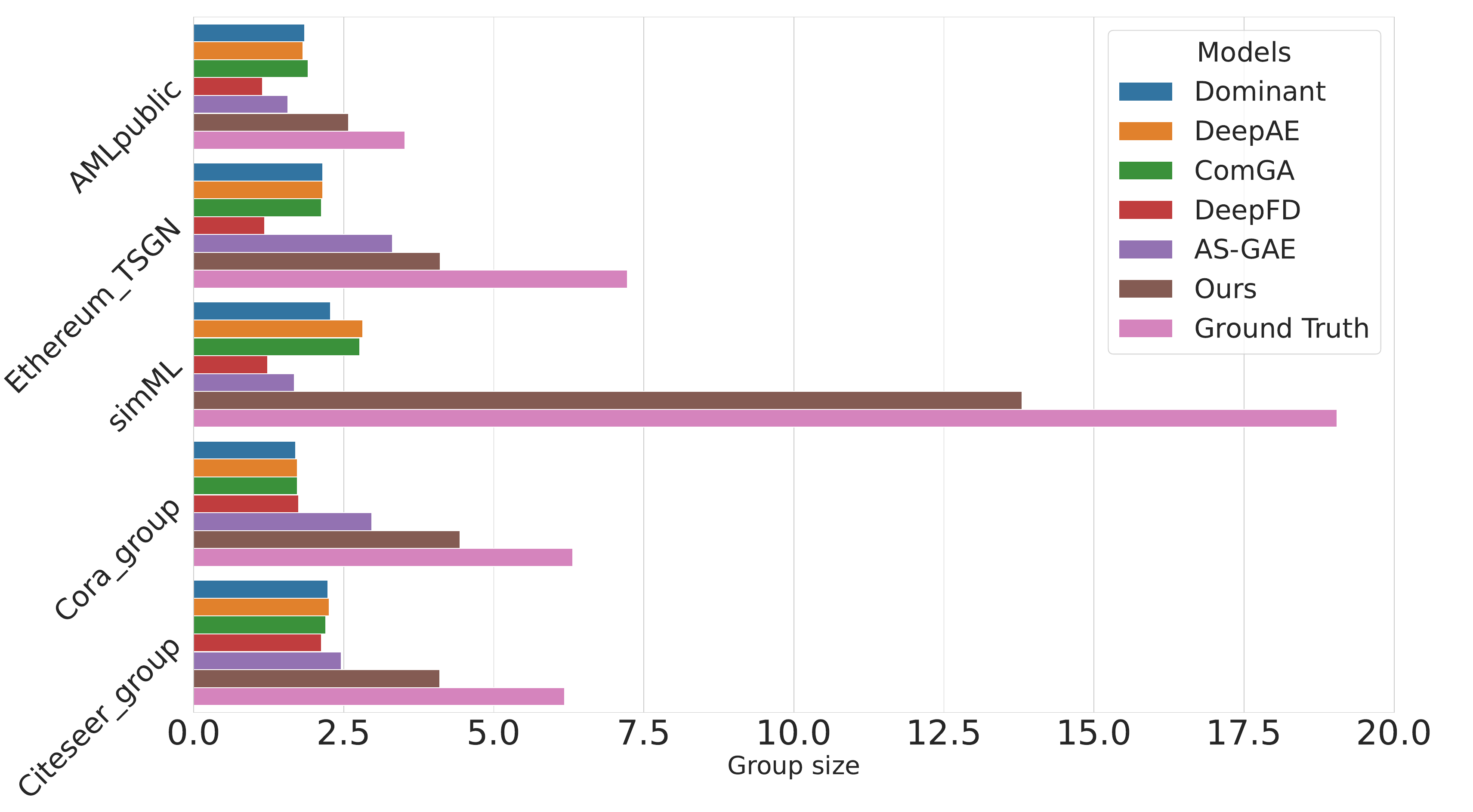}
    \caption{The average size (number of nodes) of the identified anomalous groups for each method on datasets.}
    \label{fig:group_size}
    \vspace{-0.4cm}
\end{figure}

Fig.~\ref{fig:group_size} shows the size of anomaly groups found by N-GAD methods is close to each other and not over 3, which is significantly smaller than the average size of the anomalous groups. For instance, the average size of anomaly groups of datasets except AMLPublic is more than 5. 

Although anomaly groups detected by AS-GAE  have larger sizes than the N-GAD method across all datasets and thus AS-GAE achieves higher CR values, the lower F1 scores and AUC indicate its inability to accurately distinguish anomalous groups. We believe this is due to the AS-GAE's lack of capability in extracting topology pattern information, only aggregating anomaly scores of nodes as the scores of the groups. On the contrary, the proposed framework identifies anomalous graphs with an average size that is closer to the average size of ground truth compared to any other method.

This finding supports N-GAD and Sub-GAD methods are incapable of detecting anomaly groups, thus highlighting the advantages of the proposed framework applied to Gr-GAD.

\begin{figure*}[!t]
\centering
\hfill
\begin{subfigure}{0.19\textwidth}
    \includegraphics[width=\textwidth]{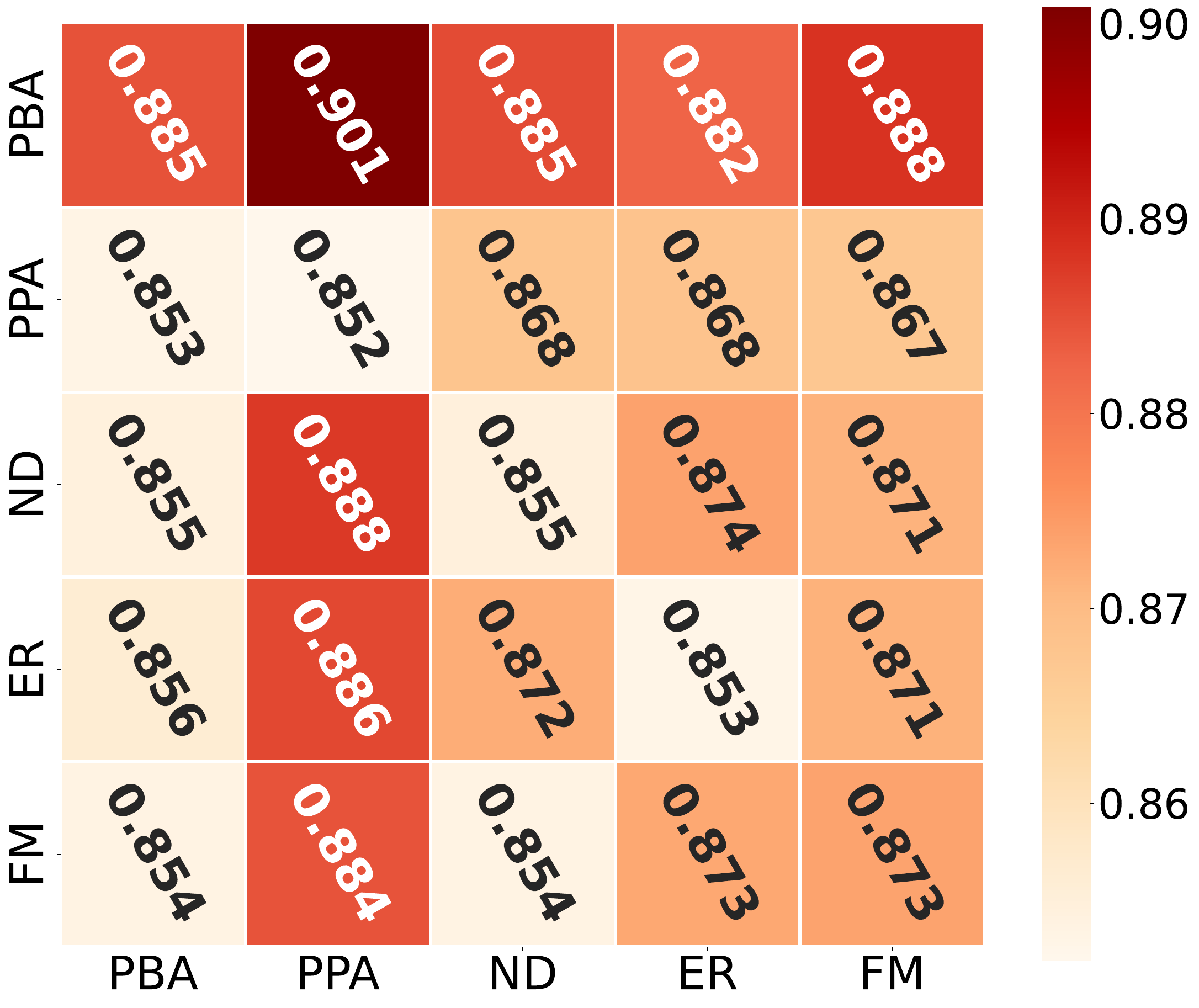}
    \caption{AMLPublic}
    \label{Aug_AMLPublic}
\end{subfigure}
\begin{subfigure}{0.19\textwidth}
    \includegraphics[width=\textwidth]{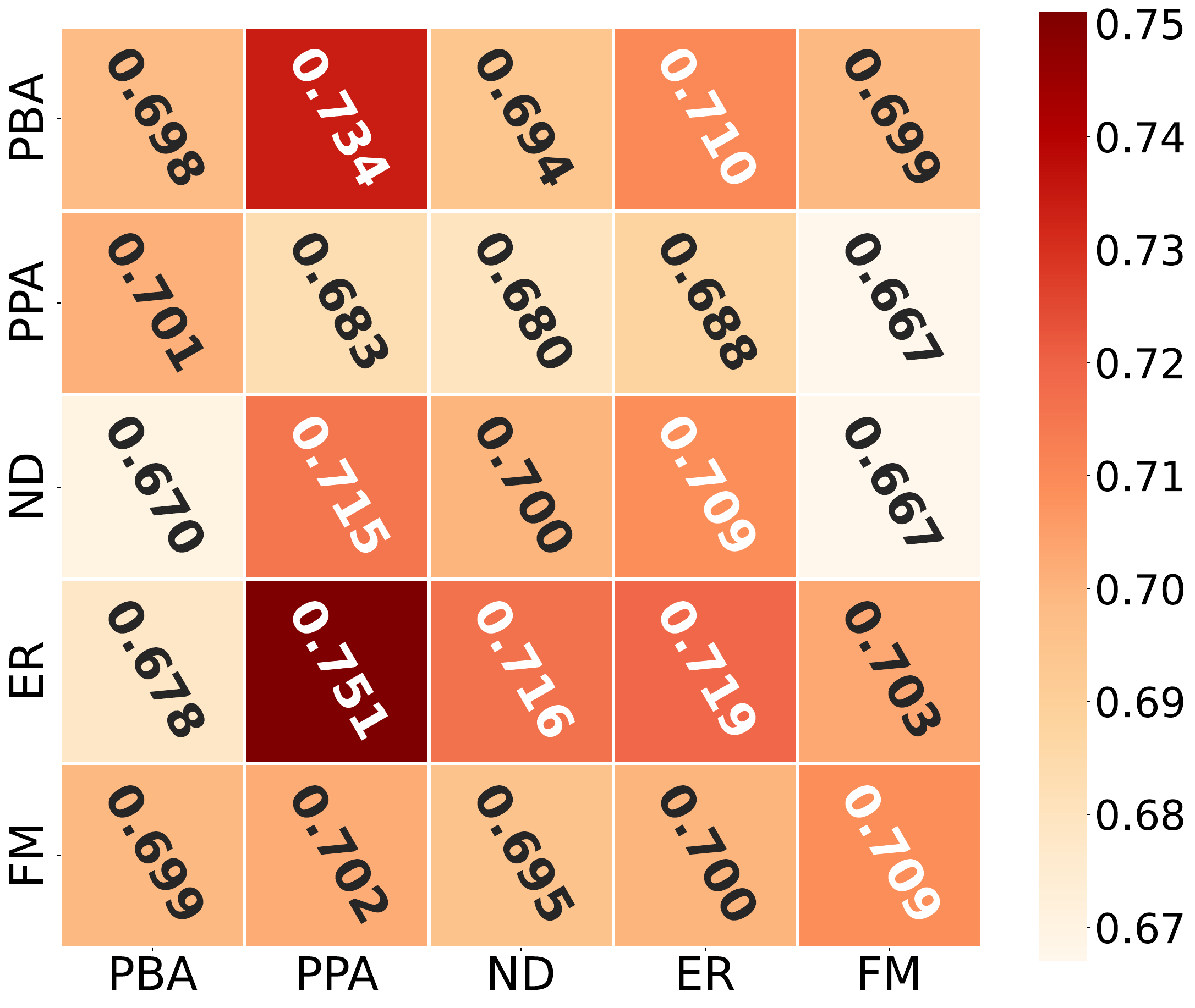}
    \caption{Ethereum-TSGN}
    \label{Aug_eth}
\end{subfigure}
\hfill
\begin{subfigure}{0.19\textwidth}
    \includegraphics[width=\textwidth]{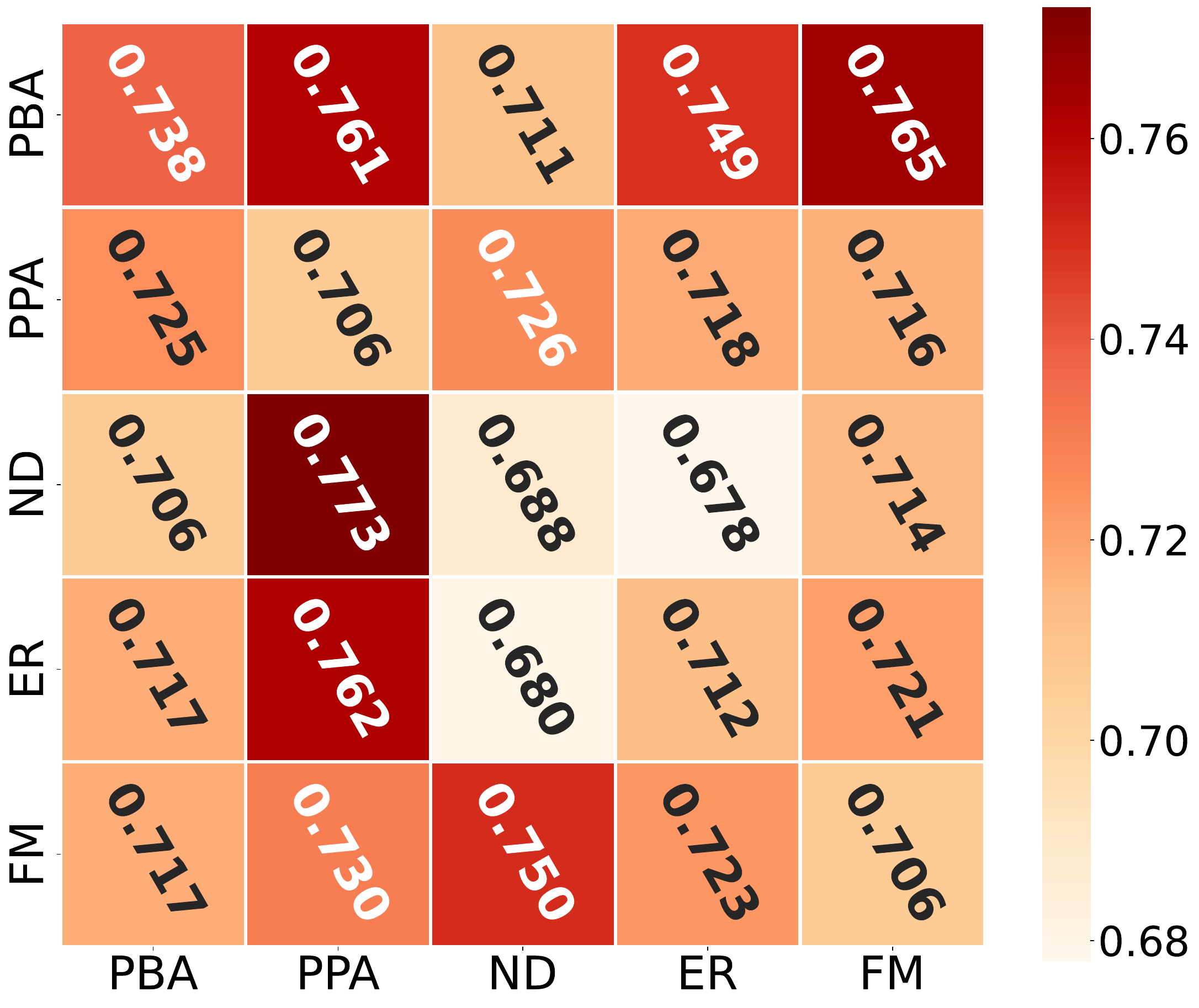}
    \caption{simML}
    \label{Aug_simML}
\end{subfigure}
\hfill
\begin{subfigure}{0.19\textwidth}
    \includegraphics[width=\textwidth]{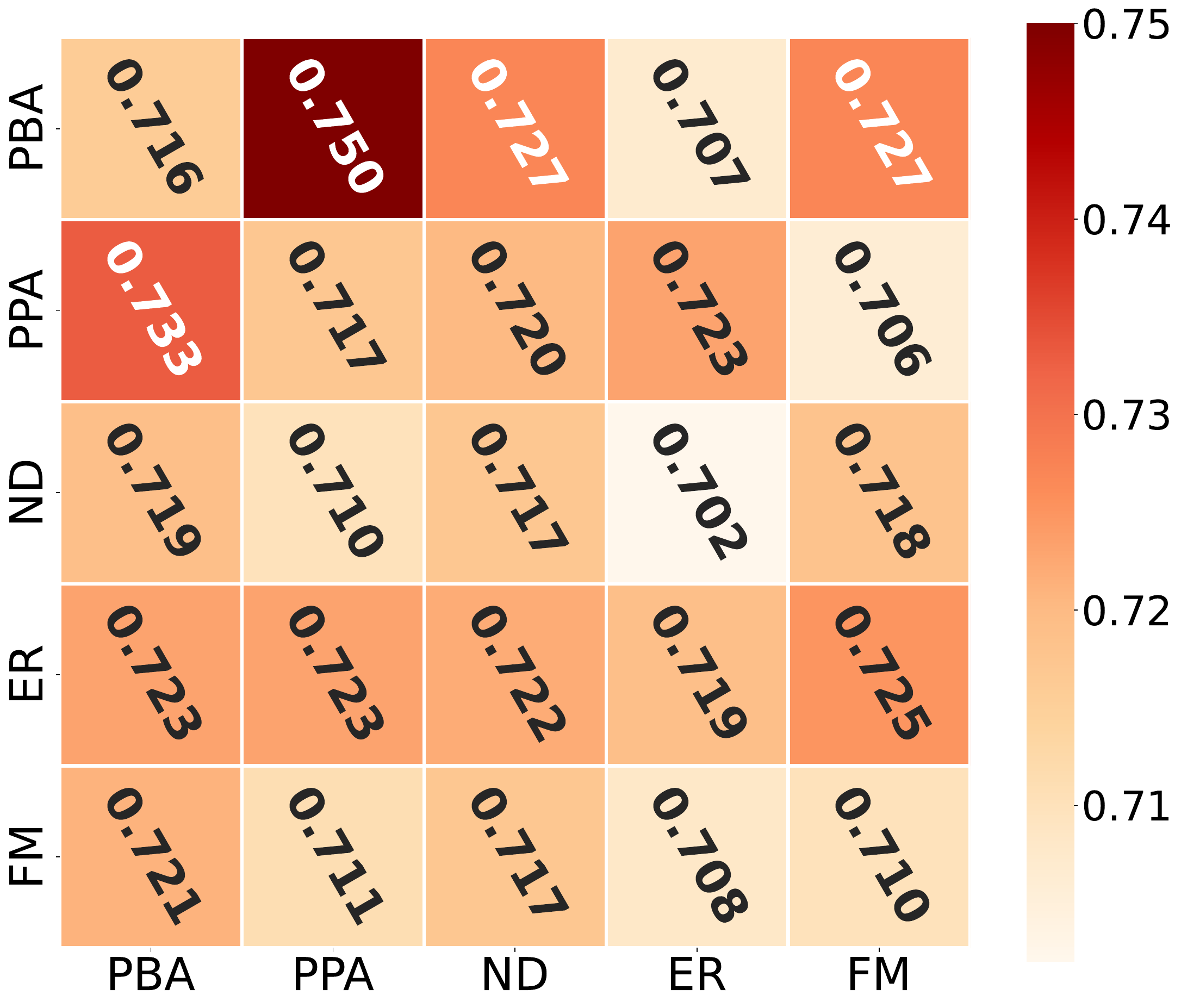}
    \caption{Cora-group}
    \label{Aug_cora}
\end{subfigure}
\hfill
\begin{subfigure}{0.19\textwidth}
    \includegraphics[width=\textwidth]{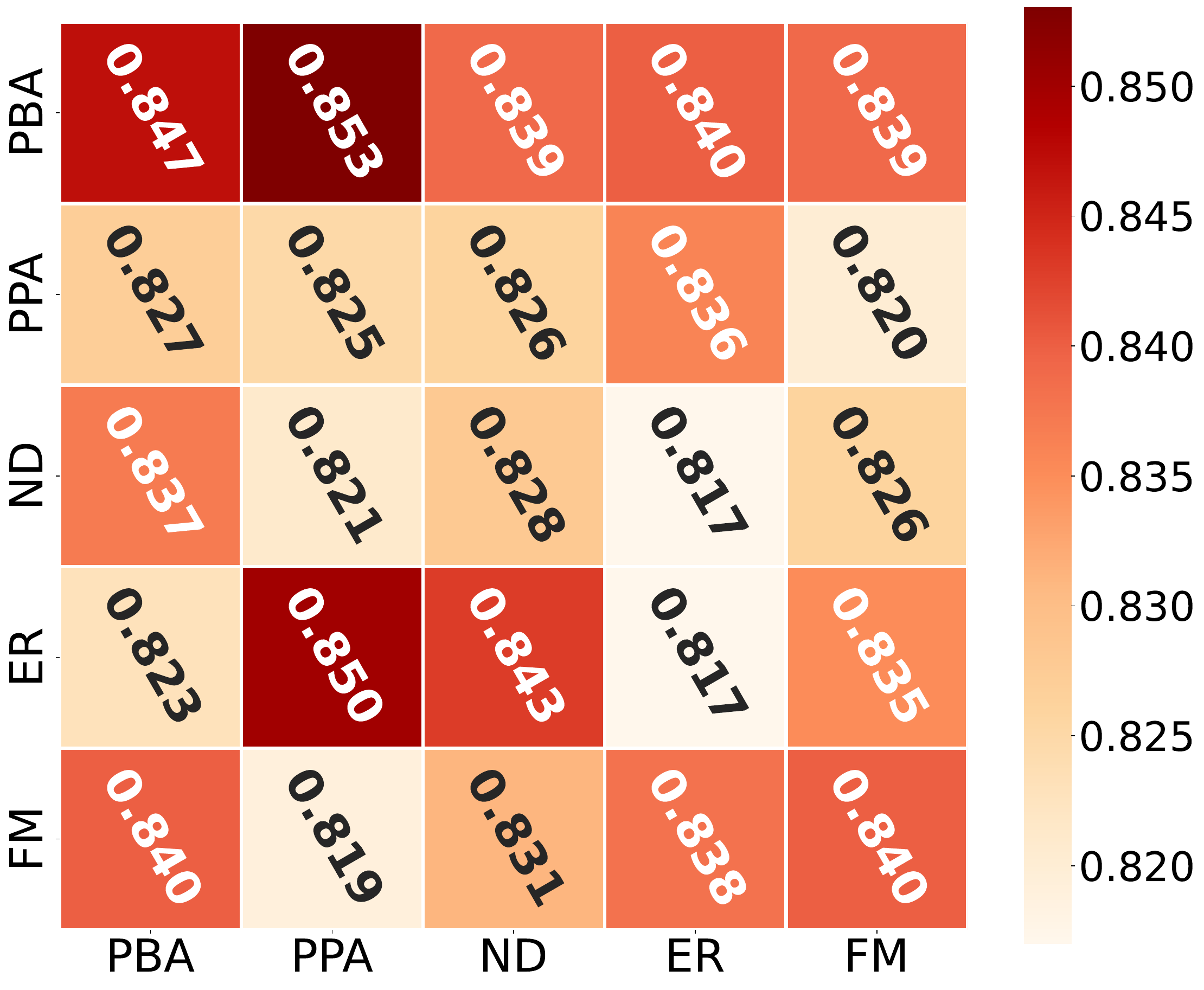}
    \caption{CiteSeer-group}
    \label{Aug_citeseer}
\end{subfigure}
\caption{Comparison of different augmentation combinations.}
\label{augmentationablation}
\end{figure*}

\begin{figure*}[!t]
\centering
\begin{subfigure}{0.19\textwidth}
    \includegraphics[width=\textwidth]{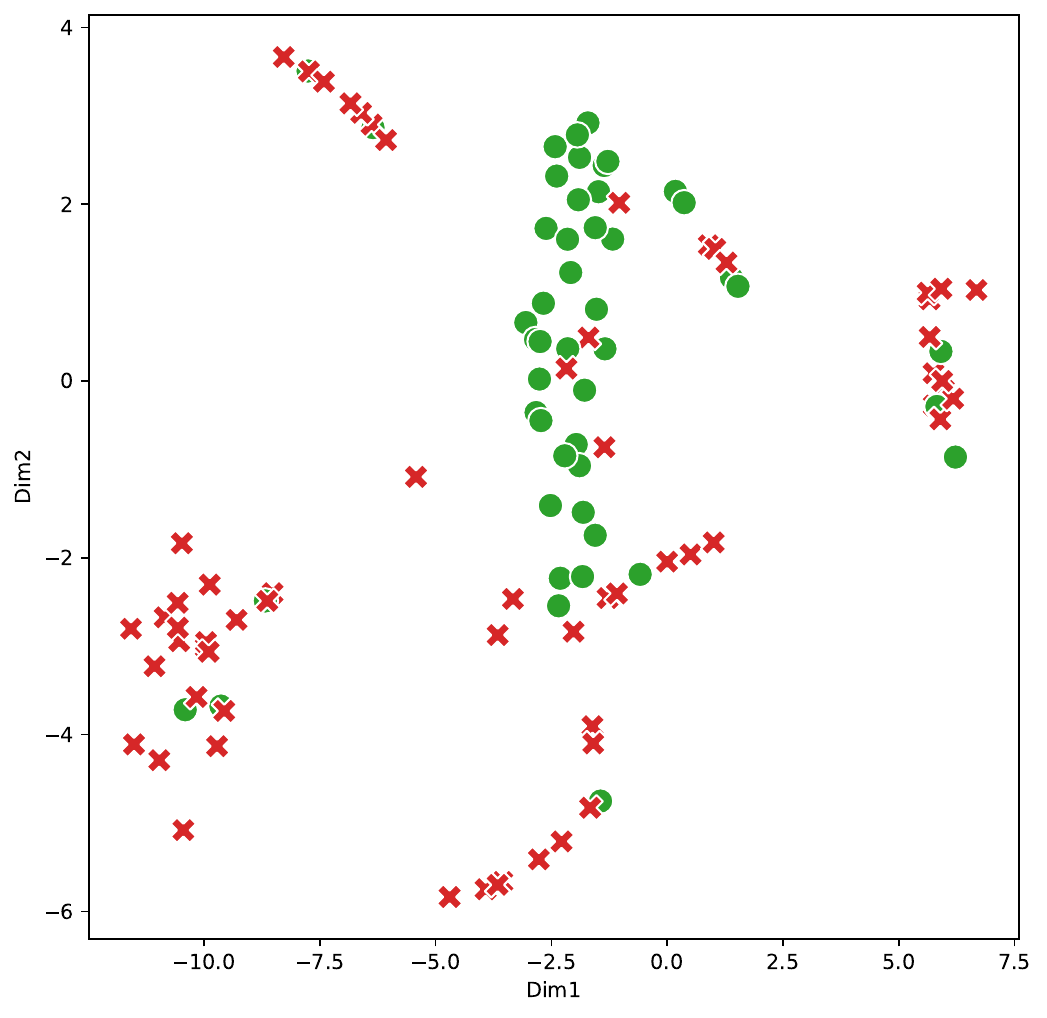}
    \caption{tSNE simML}
    \label{tSNE-simML}
\end{subfigure}
\hfill
\begin{subfigure}{0.19\textwidth}
    \includegraphics[width=\textwidth]{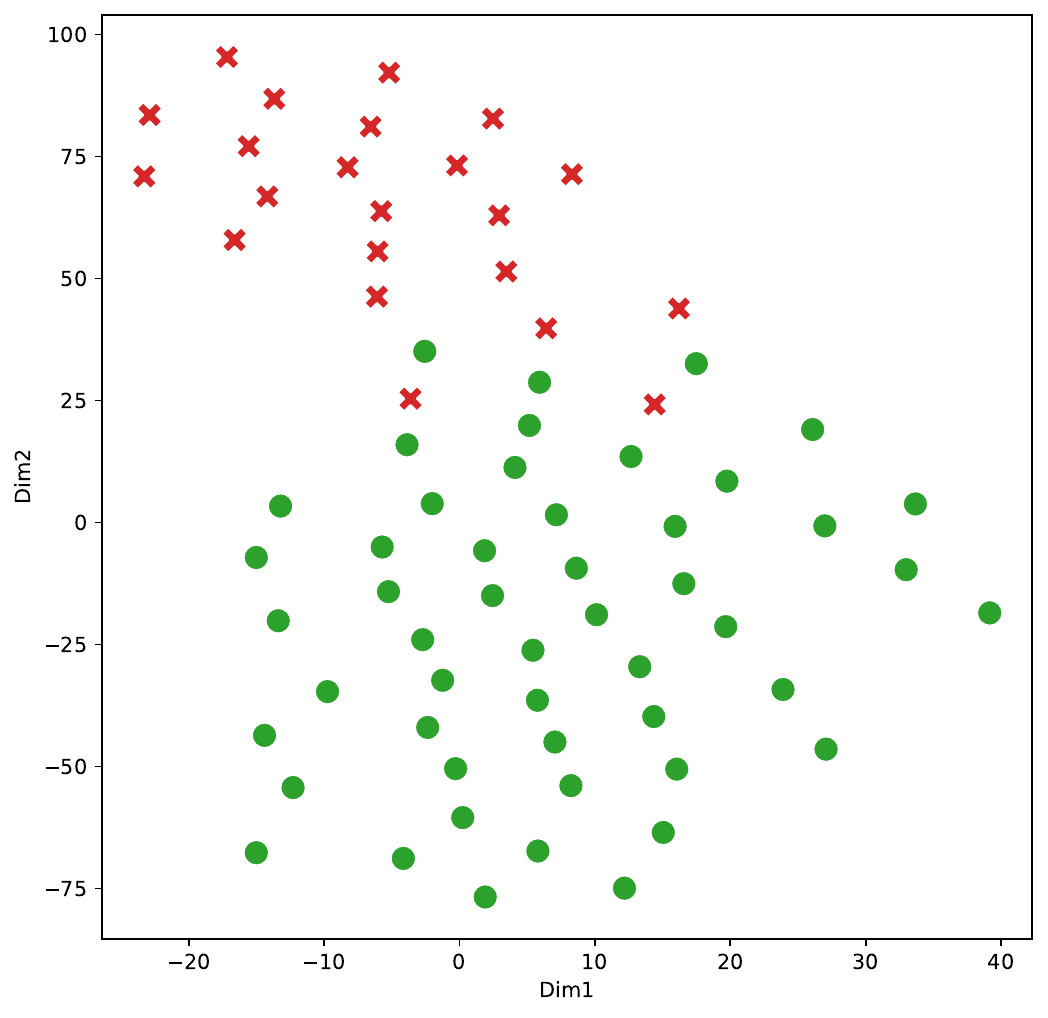}
    \caption{tSNE Cora-g}
    \label{tSNE-Cora}
\end{subfigure}
\hfill
\begin{subfigure}{0.19\textwidth}
    \includegraphics[width=\textwidth]{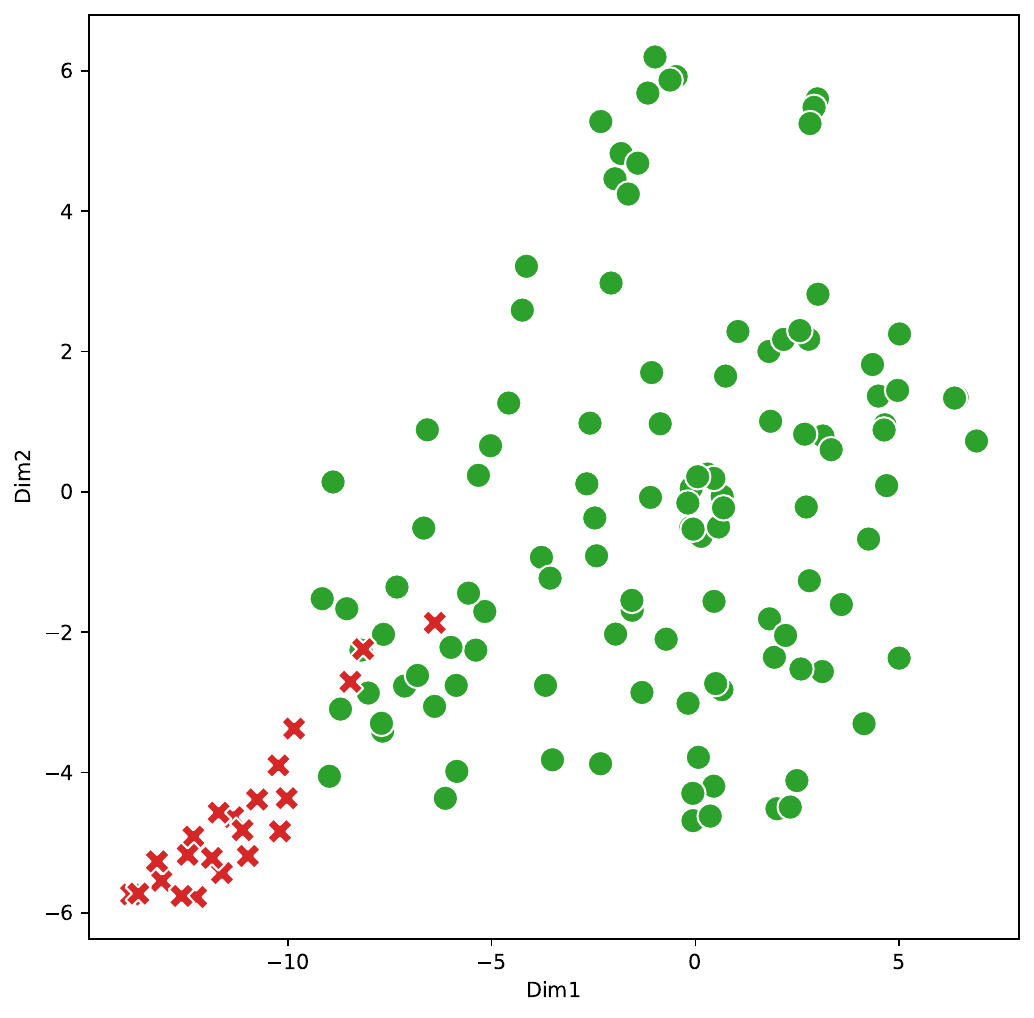}
    \caption{tSNE CiteSeer-g}
    \label{tSNE-Citeseer}
\end{subfigure}
\hfill
\begin{subfigure}{0.19\textwidth}
    \includegraphics[width=\textwidth]{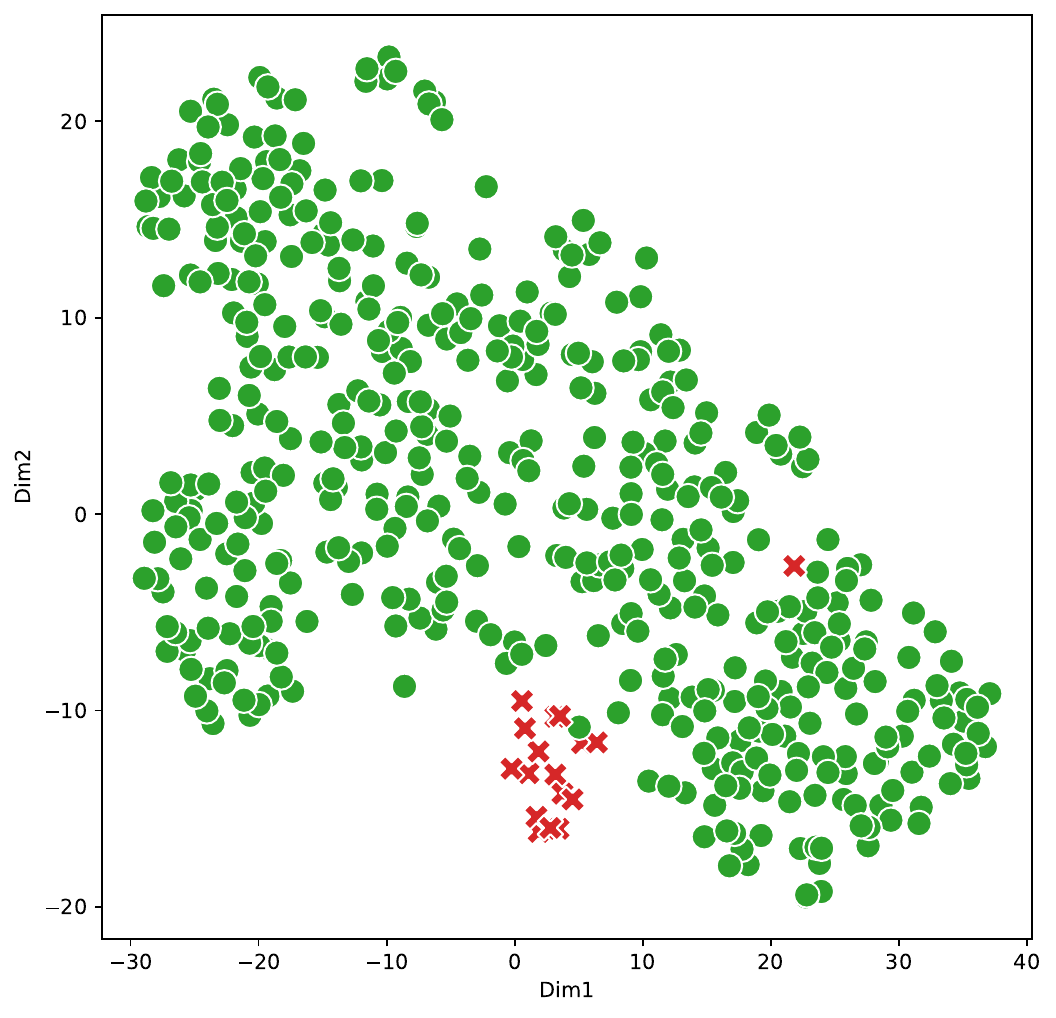}
    \caption{tSNE AMLPublic}
    \label{tSNEAMLPublic}
\end{subfigure}
\hfill
\begin{subfigure}{0.19\textwidth}
    \includegraphics[width=\textwidth]{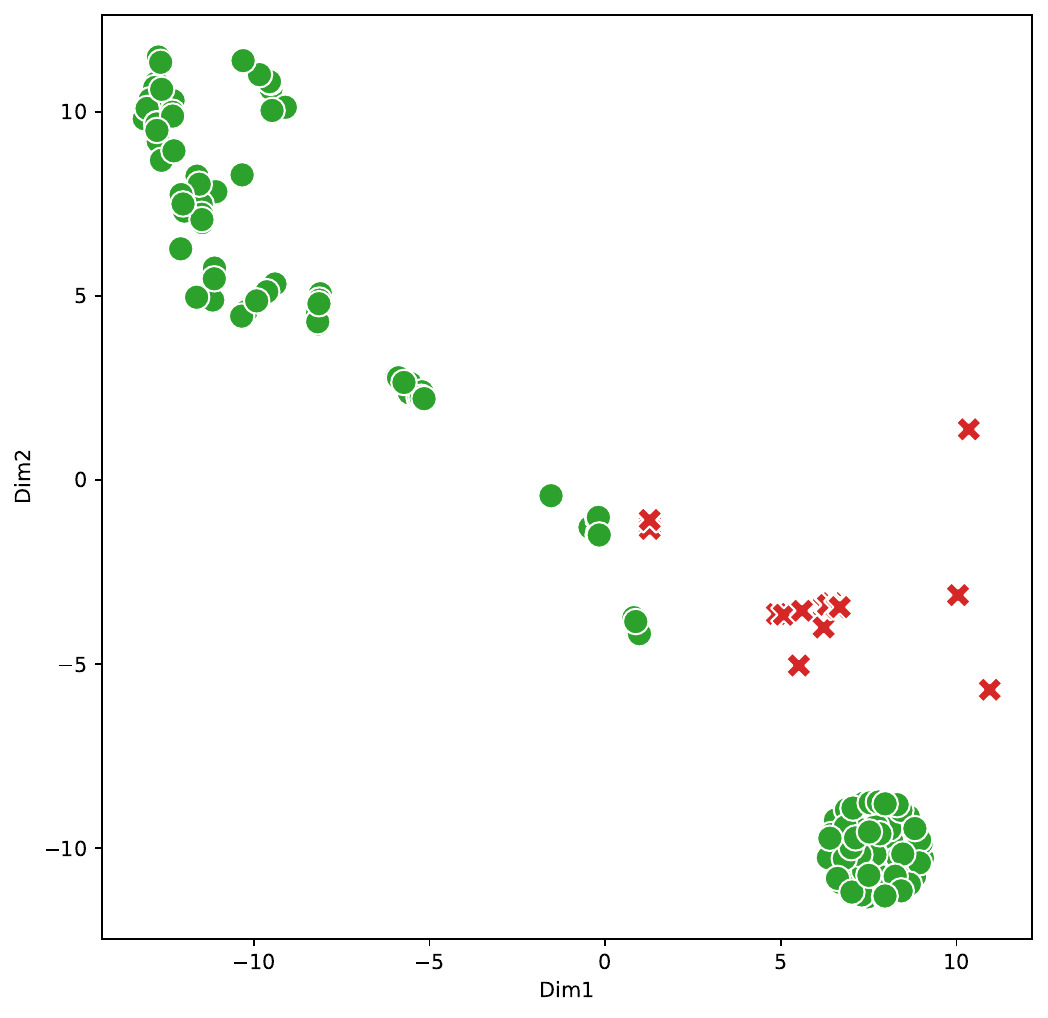}
    \caption{tSNE Eth}
    \label{tSNEEthereumTSGN}
\end{subfigure}
\caption{Visualizations on datasets. Red and green nodes are embeddings of anomaly and normal groups respectively.}
\label{fig:visualization}
\end{figure*}

\begin{figure*}[!t]
    \begin{subfigure}{0.19\textwidth}
        \includegraphics[width=\textwidth]{Pictures/example_graph.pdf}
        \caption{Example Graph}
        \label{example graph1}
    \end{subfigure}
    \hfill
    \begin{subfigure}{0.19\textwidth}
        \includegraphics[width=\textwidth]{Pictures/dominant_example1.pdf}
        \caption{DOMINANT}
        \label{Dominant example1}
    \end{subfigure}
    \hfill
    \begin{subfigure}{0.19\textwidth}
        \includegraphics[width=\textwidth]{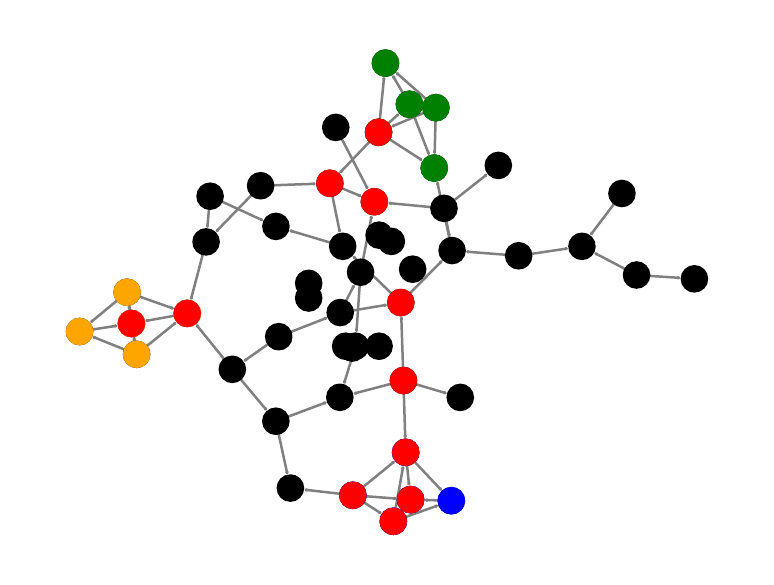}
        \caption{DeepAE}
        \label{DeepAE example1}
    \end{subfigure}
    \hfill
    \begin{subfigure}{0.19\textwidth}
        \includegraphics[width=\textwidth]{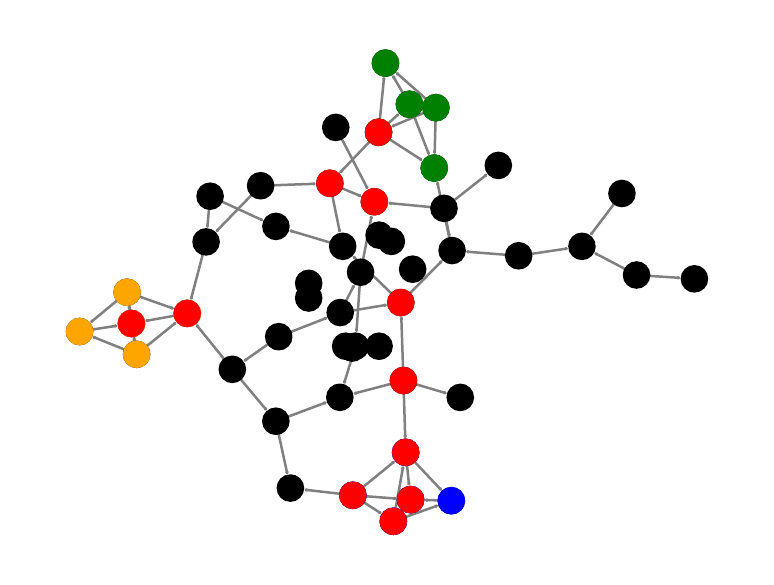}
        \caption{ComGA}
        \label{ComGA example1}
    \end{subfigure}
    \hfill
    \begin{subfigure}{0.19\textwidth}
        \includegraphics[width=\textwidth]{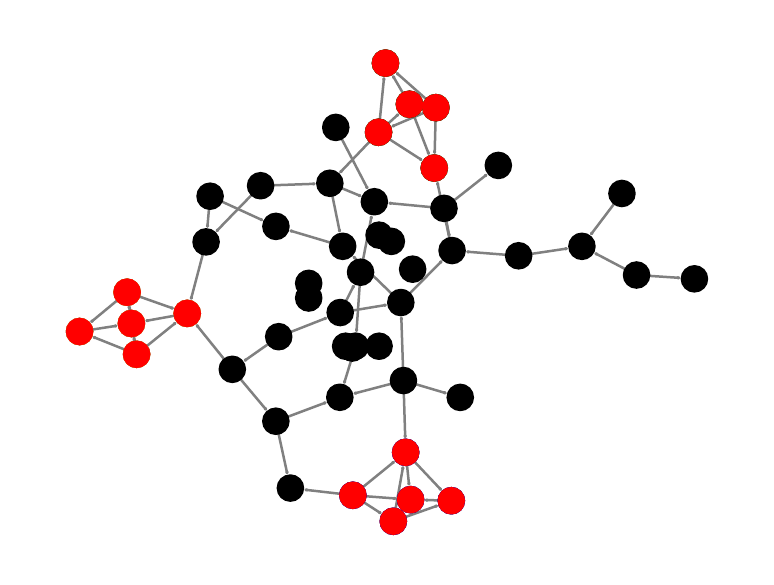}
        \caption{MH\-GAE}
        \label{MHGAE example}
    \end{subfigure}
\label{fig:example}
\caption{Performance comparison of GAE-based anomaly detection methods on example graph.}
\vspace{-0.4cm}
\end{figure*}

\subsection{Ablation Studies}
\label{Exp: ablation studies}

To demonstrate the effectiveness of the proposed MH-GAE and TPGCL, we design the following ablation studies. 

\subsubsection{Ablation study of MH-GAE}

To verify MH-GAE's ability to capture long-range inconsistency, we compare performances of various optimization objects by replacing adjacent matrix $A$ with different matrics and collect the proposed method's CR scores under different matrics as shown in Table.~\ref{Ablation study of MH-GAE}. 

\begin{table}[!h]
\centering
\caption{Comparison of matrix for MH-GAE.}
\label{Ablation study of MH-GAE}
\begin{tabular}{@{}c|c|c|c|c|c@{}}
\toprule[0.6mm]
Matrix & $A$ & $A^3$ & $A^5$ & $A^7$ & $\tilde{A}$ \\ \midrule
Ethereum-TSGN & 0.693 & 0.692 & \underline{$0.759$} & 0.735 & \bm{$0.81$} \\ \midrule
AMLPublic & 0.851 & 0.860 & 0.854 & \underline{$0.888$} & \bm{$0.890$} \\ \midrule
simML & 0.831 & 0.826 & \bm{$0.842$} & 0.835 & \underline{$0.839$} \\ \midrule
Cora-group & 0.840 & 0.870 & \underline{$0.920$} & 0.90 & \bm{$0.933$} \\ \midrule
CiteSeer-group & 0.692 & 0.689 & \bm{$0.747$} & 0.702 & \underline{$0.724$} \\ \bottomrule[0.6mm]
\end{tabular}
\end{table}

\begin{table}[!h]
\centering
\caption{Ablation study of TPGCL.}
\label{Ablation study of TPGCL}
\begin{tabular}{@{}c|c|c@{}}
\toprule[0.6mm]
Dataset & TP-GrGAD w/o TPGCL & TP-GrGAD \\ \midrule
Ethereum-TSGN & 0.402±0.03 & 0.734±0.04 \\ \midrule
AMLPublic & 0.639±0. & 0.901±0.  \\ \midrule
simML & 0.426±0.01 & 0.761±0.04 \\ \midrule
Cora-group & 0.578±0.01 & 0.750±0.02 \\ \midrule
CiteSeer-group & 0.675±0.01 & 0.853±0.01 \\ \bottomrule[0.6mm]
\end{tabular}
\vspace{-0.3cm}
\end{table}

The results are shown in Table.~\ref{Ablation study of MH-GAE} indicate MH-GAE always achieves the worst performance when the objective matrix is $A$ or $A^3$. Since MH-GAE only captures inconsistency within one- or three hops, which is obviously not enough to measure group-level anomaly. Conversely, $A^5$, $A^7$ and $\tilde{A}$ can help improve the performance of MH-GAE and achieve the best and second-best performance. This observation demonstrates the significance of capturing long-range inconsistency for Gr-GAD.

\subsubsection{Ablation study of PPA and PBA}

We compare the proposed PPA and PBA with the three most commonly used augmentations: \textbf{Node Dropping (ND)}, \textbf{Edge Removing (ER)}, and \textbf{Feature Masking (FM)}. Note ND and ER perturb the structures of candidate groups by removing nodes and edges, which will damage the intrinsic topology patterns.

\subsubsection{Ablation study of TPGCL}

To validate the effectiveness of TPGCL, we remove the TPGCL component from the framework and directly input candidate groups into the unsupervised outlier detection method \cite{ECOD} to obtain anomaly scores and evaluate the performance. Each candidate group was represented as a feature vector, obtained by averaging the feature vectors of each node within the candidate group. We compared the F1-score of these two variants of our proposed framework and presented the results in Table.~\ref{Ablation study of TPGCL}. The table indicates that after removing the TPGCL component, the performance of TP-GrGAD significantly deteriorates, and it becomes ineffective in distinguishing anomalous groups. This finding suggests that the embeddings generated by TPGCL, which contain topology pattern information, are crucial for effectively distinguishing anomalous groups.

\subsection{Visualization}
\label{Exp:visualization}

To demonstrate TPGCL can provide well embeddings for the input candidate groups for discrimination, we use t-SNE \cite{tSNE} to visualize the embeddings learned by TPGCL and show the result in Fig.~\ref{fig:visualization}. Despite a few failed cases, most embeddings belonging to the same class tend to be clustering, which means learned embeddings are high quality and easy to distinguish. For example, as shown in Fig.~\ref{tSNE-Citeseer}, anomaly groups' embeddings (red ``x") are close to each other and far away from normal groups' embeddings (green nodes).  A similar situation can also be observed in  Fig.~\ref{tSNEAMLPublic}. 

To demonstrate MH-GAE's ability to capture long-range inconsistency, we generate an example graph shown in Fig.~\ref{example graph1}, which contains three anomaly groups highlighted in blue, green, and orange colors. We employ DOMINANT, DeepAE, ComGA, and MH-GAE in this example graph and visualize their predicted results, highlighted in red color. In Fig.~\ref{Dominant example1}, Fig.~\ref{DeepAE example1}, Fig.~\ref{ComGA example1}, some connected components with sizes ranging from 1 to 5 are detected, while nodes hidden deep within the anomalous groups remain undetected, thus validating our analysis about N-GAD methods lack the capability to capture long-range inconsistency, consequently impacting Gr-GAD. Conversely, in Fig.~\ref{MHGAE example}, MH-GAE can successfully capture the entire anomalous group, showcasing its superiority in Gr-GAD.

\section{Conclusion}\label{Conclusion}
In this paper, we addressed the challenging task of Group-level Graph Anomaly Detection (Gr-GAD). We indicate this new task requires anomaly detection methods capable of capturing long-range inconsistency and topology patterns information. To meet these requirements, we introduce a novel unsupervised framework comprising a new variant of Graph AutoEncoder (GAE) called Multi-Hop Graph AutoEncoder (MH-GAE) and a novel Graph Convolutional Learning (GCL) method named Topology Pattern-based Graph Contrastive Learning (TPGCL). Additionally, we provide theoretical proof, from the perspective of mutual information and Graph Information Bottleneck (GIB), to demonstrate that topology patterns can serve as clues for detecting anomalous groups. Experimental results on both real-world and synthetic datasets validate the effectiveness and superiority of our proposed framework.

\newpage
\bibliographystyle{IEEEtranN}
\bibliography{ref}
\end{document}